%% file: arxiv_main.tex
\documentclass{article}

\usepackage[english]{babel}
\usepackage{geometry}
\usepackage[round]{natbib}
\input{math_commands}

\newcommand{\footremember}[2]{%
    \footnote{#2}
    \newcounter{#1}
    \setcounter{#1}{\value{footnote}}%
}
\newcommand{\footrecall}[1]{%
    \footnotemark[\value{#1}]%
}

\usepackage[utf8]{inputenc} 
\usepackage[T1]{fontenc}    
\usepackage[colorlinks,urlcolor=blue,citecolor=blue,linkcolor=blue]{hyperref}       
\usepackage{url}            
\usepackage{booktabs}       
\usepackage{amsfonts}       
\usepackage{nicefrac}       
\usepackage{microtype}      
\usepackage{xcolor}         
\usepackage{subfigure}
\allowdisplaybreaks[4]

\title{Variance Reduced Halpern Iteration \\
for Finite-Sum Monotone Inclusions}
\author{Xufeng Cai\footremember{equal}{Equal contribution.}\footremember{wisc}{Department of Computer Sciences, University of Wisconsin-Madison. XC (\href{mailto:xcai74@wisc.edu}{xcai74@wisc.edu}), JD (\href{mailto:jelena@cs.wisc.edu}{jelena@cs.wisc.edu}).}
\and Ahmet Alacaoglu\footrecall{equal} \footremember{wid}{Wisconsin Institute for Discovery, University of Wisconsin–Madison. \href{mailto:alacaoglu@wisc.edu}{alacaoglu@wisc.edu}.}
\and Jelena Diakonikolas\footrecall{wisc}
}
\date{}

\begin{document}

\maketitle
\begin{abstract}%
Machine learning approaches relying on such criteria as adversarial robustness or multi-agent settings have raised the need for solving game-theoretic equilibrium problems. Of particular relevance to these applications are methods targeting finite-sum structure, which generically arises in empirical variants of learning problems in these contexts. Further, methods with computable approximation errors are highly desirable, as they provide verifiable exit criteria. Motivated by these applications, we study finite-sum monotone inclusion problems, which model broad classes of equilibrium problems. Our main contributions are variants of the classical Halpern iteration that employ variance reduction to obtain improved complexity guarantees in which $n$ component operators in the finite sum are ``on average'' either cocoercive or Lipschitz continuous and monotone, with parameter $L$. The resulting oracle complexity of our methods, which provide guarantees for the last iterate and for a (computable) operator norm residual, is $\widetilde{\mathcal{O}}( n + \sqrt{n}L\varepsilon^{-1})$, which improves upon existing methods by a factor up to $\sqrt{n}$. This constitutes the first variance reduction-type result for general finite-sum monotone inclusions and for more specific problems such as convex-concave optimization when operator norm residual is the optimality measure. We further argue that, up to poly-logarithmic factors, this complexity is unimprovable in the monotone Lipschitz setting; i.e., the provided result is near-optimal.
\end{abstract}

\section{Introduction}\label{sec:intro}
We study finite-sum monotone inclusion problems, where the goal is to find $\vu_* \in \sR^d$ such that 
\begin{equation}\label{eq: prob1}
    \mathbf{0} \in F(\vu_*) + G(\vu_*), \tag{MI}
\end{equation}
and where $F(\vu)\colon\mathbb{R}^d\to\mathbb{R}^d$ is single-valued, monotone, and Lipschitz, and $G(\vu)\colon\mathbb{R}^d\rightrightarrows\mathbb{R}^d$ is maximally monotone and possibly multi-valued. 
We consider in particular the setting where $F$ possesses the \emph{finite-sum} structure, meaning that it is expressible as $ F(\vu) = \frac{1}{n}\sum_{i=1}^n F_i(\vu).$

As is standard, we assume access to (\emph{i}) the resolvent of $\eta G$ for $\eta > 0,$ meaning that for any $\vu$ we can find $\Bar{\vu}$ such that $ \vu - \Bar{\vu} \in \eta G(\Bar{\vu})$ (generalizing the \emph{proximal operator}); and (\emph{ii}) oracle evaluations of component operators $F_i$. We measure oracle complexity of an algorithm in terms of oracle evaluations of component operators $F_i$ and the resolvent operator of $\eta G.$

The considered finite-sum settings are widespread in machine learning; see e.g.,~\citet{johnson2013accelerating, defazio2014saga, schmidt2017minimizing,gower2020variance}.
While finite-sum \emph{minimization} problems are well-studied, recent applications in areas such as generative adversarial networks~\citep{goodfellow2014generative}, robust machine learning~\citep{madry2018towards}, and multi-agent reinforcement learning~\citep{zhang2021multi} require solving more general equilibrium problems. Such tasks are neatly unified under the umbrella of monotone inclusion~\eqref{eq: prob1}, which has a rich history within optimization theory and operations research~\citep{facchinei2007finite}.

An important special case of~\eqref{eq: prob1} is the \emph{variational inequality (VI)} problem defined as below, where $G$ is the subdifferential of the indicator function of a closed convex set $C \subseteq \mathbb{R}^d$: 
\begin{equation}\label{eq: prob_vi}
    \text{find~} \vu_\ast \in C \text{~such that~} \langle F(\vu_\ast), \vu - \vu_\ast \rangle \geq 0, ~~~\forall \vu \in C. \tag{VI}
\end{equation}
A more specialized template that results in interesting examples of VI and monotone inclusion problems is min-max optimization $\min_\vx \max_\vy f(\vx) - g(\vy) + \Phi(\vx, \vy),$ where $\Phi$ is convex-concave with Lipschitz gradients and $f, g$ are proper convex lower semicontinuous (l.s.c.). 
This maps to~\eqref{eq: prob1} by setting 
$    \vu=\binom{\vx}{\vy},$ $F(\vu) = \binom{\nabla_\vx \Phi(\vx, \vy)}{-\nabla_\vy \Phi(\vx, \vy)},$ and $G(\vu) = \binom{\partial f(\vx)}{\partial g(\vy)}$.

\paragraph{Optimality measures.} A standard optimality measure for solving~\eqref{eq: prob_vi} problems  is the \emph{duality gap}~\citep{facchinei2007finite}, defined as
\begin{equation}\label{eq: gap_def}
    \mathrm{Gap}(\vu) = \max_{\vv \in C}\langle F(\vv), \vu - \vv \rangle. \tag{Gap}
\end{equation}
However,~\eqref{eq: gap_def} has significant drawbacks especially when the domain $C$ is unbounded, which is often the case. A common way to get around this is to use the \emph{restricted duality gap}~\citep{nesterov2007dual}, which requires the identification of a compact set including the iterates. However, such a set generally affects the constants in the convergence bounds, and the restricted duality gap is not as interpretable as the duality gap. Additional drawbacks are that (\emph{i}) neither of these measures is efficiently computable in general, (\emph{ii}) the guarantees for these measures are typically obtained for an \emph{average} or the \emph{best} iterate, and (\emph{iii}) such duality gap guarantees are generally not possible to obtain outside of monotone operator (convex-concave in the case of min-max optimization) settings.  

An alternative optimality measure, which we focus on in this paper and argue to be more general than the duality gap, is the \emph{residual} defined as
\begin{equation}\label{eq:res}
    \mathrm{Res}_{F+G}(\vu) = \| F(\vu) + g(\vu)\|, \tag{Res}
\end{equation}
where $g(\vu)\in G(\vu)$ and hence $\displaystyle \mathrm{dist}(F(\vu) + G(\vu), \mathbf{0}) = \min_{g(\vu) \in G(\vu)}\|F(\vu) + g(\vu)\| \leq \mathrm{Res}_{F+G}(\vu)$. 
The complexity results for~\eqref{eq:res} can be translated to duality gap but not vice versa (see, e.g.,~\citet[Section 1.2]{diakonikolas2020halpern}). Moreover, this measure is in most cases computable since the algorithms typically have access to $F(\vu) + g(\vu)$, which will become more clear in the sequel. Further, all our results are for the \emph{last iterate} and the residual error guarantees are possible even for some classes of structured non-monotone operators.  

\paragraph{Context.} When restricted to the complexity results in terms of the duality gap, there exist optimal algorithms for finite-sum monotone VIs~\citep{palaniappan2016stochastic,carmon2019variance,alacaoglu2021stochastic}. These works show how to take advantage of the finite-sum structure with variance reduction techniques to improve the complexity compared to deterministic algorithms such as the extragradient method. However, as described earlier, these results do not translate to guarantees on the residual. Even more, we cannot expect these existing variance reduced algorithms to have optimal rates for the residual, since in the deterministic case  they reduce to algorithms that are known to be suboptimal for residual guarantees~\citep{golowich2020last}.

On the other hand, even in the deterministic case, algorithms that are optimal for the residual error (in terms of oracle/iteration complexity) were obtained only recently \citep{diakonikolas2020halpern,Yoon2021OptimalGradientNorm}. These results are based on variants of the classical Halpern iteration \citep{halpern1967fixed} developed for solving fixed point equations with nonexpansive maps. 
Despite further developments relaxing the initial assumptions on $F$ and $G$~\citep{tran2021halpern,cai2022accelerated,cai2022acceleratedsingle,kovalev2022first} and addressing stochastic approximation settings~\citep{cai2022stochastic, chen2022near}, none of the existing results harness the finite-sum structure nor lead to the $\sqrt{n}$ improvements expected from variance reduction approaches in such settings.

\paragraph{Our contributions.}
We obtain the first variance-reduced complexity results for standard classes of \eqref{eq: prob1} problems that can lead to a $\sqrt{n}$ improvement compared to methods without variance reduction, in line with similar results obtained for the less general minimization and min-max problems (focusing only on the duality gap guarantees) in the existing literature. In particular: 
\begin{itemize}[leftmargin=*]
    \item When $F$ satisfies average $\frac{1}{L}$-cocoercivity (see Assumption~\ref{asp: asp3}), we obtain an algorithm with   oracle complexity $\widetilde{\mathcal{O}}( n + \sqrt{n}L{\varepsilon^{-1}})$. 
    To obtain this result, we incorporate recursive variance reduction~\citep{li2021page} into constrained one-step Halpern iteration. While a similar strategy has been employed in \citet{cai2022stochastic} to address stochastic approximation (infinite sum) settings, their analysis is more complicated and their oracle complexity $\mathcal{O}(\varepsilon^{-3})$
    is strictly worse than ours whenever $n$ is not too large (roughly, when $n = o({1}/{\varepsilon^3})$).
    \item When $F$ is monotone and $L$-Lipschitz in expectation (see Assumption~\ref{asp: asp2}), 
    we obtain an algorithm with oracle complexity $\widetilde{\mathcal{O}}( n + \sqrt{n}L{\varepsilon^{-1}})$. 
    This result is enabled by a variant of Halpern iteration employing inexact resolvent evaluations of $\eta(F + G)$ for $\eta > 0.$ While this strategy is similar to the approach taken in \citet{diakonikolas2020halpern} to address corresponding settings without the finite-sum considerations,  unlike their work our result is enabled by employing a stochastic variance reduced algorithm from \citet{alacaoglu2021stochastic}. A 
    critical difference in the analysis is that we carry out a stochastic error analysis with a new inexactness criterion, due to the randomized nature of the inner algorithm. To obtain our result, we generalize the analysis for strongly monotone inclusion from \citet{alacaoglu2021stochastic}  to the composite form \eqref{eq: prob1} with a maximal monotone $G$. 
\end{itemize}

\paragraph{Further related work.} 
Monotone inclusion and fixed point problems with finite-sum cocoercive operators have been the focus of study in several recent papers.~\citet{davis2023variance} presented a possible speedup with variance reduction for root-finding problems with average cocoercivity only w.r.t.\ the solution point, but requiring additional quasi strong monotonicity assumption.
\citet{loizou2021stochastic,gorbunov2022stochastic,beznosikov2023stochastic} used similar assumptions to solve the VI problems and provided convergence results for stochastic gradient descent-ascent
methods. A more general expected residual assumption was considered in~\citet{choudhury2023single}, but this work only proved $\mathcal{O}(\varepsilon^{-4})$ complexity for residual norm guarantees in our setting, which is suboptimal for finite-sum monotone problems when $n = o(\varepsilon^{-4})$.
Similar less desirable $\mathcal{O}(\varepsilon^{-2})$ complexity on the residual norm was also obtained in~\citet{morin2022cocoercivity} with component cocoercivity.~\citet{tran2023randomized} considered random coordinate methods for root finding problems with cocoercivity, which is  a different setting that does not improve upon the overall complexity over deterministic algorithms.
For the finite-sum monotone Lipschitz setting,~\citet{johnstone2021stochastic} obtained $\mathcal{O}(\varepsilon^{-8})$ complexity for a generalized version of the residual and left open the problem of obtaining a better complexity for the residual norm by a stochastic method in this setting, which our results address.

\section{Preliminaries}\label{sec: notat_asp}
We consider the $d$-dimensional real space $(\sR^d, \|\cdot\|)$, where $\|\cdot\|$ is the $\ell_2$ norm. 
We say that an operator $T\colon \mathbb{R}^d \to \mathbb{R}^d$ is 
(\emph{i}) \emph{monotone} if for $\forall \vu, \vv \in \mathbb{R}^d$: $\innp{ T(\vu) - T(\vv), \vu - \vv} \geq 0$; (\emph{ii})
$L_F$-\emph{Lipschitz} if for $\forall \vu, \vv \in \mathbb{R}^d$: $\| T(\vu) - T(\vv) \| \leq L_F \| \vu - \vv \|$; (\emph{iii})
$\frac{1}{L}$-\emph{cocoercive} if for $\forall \vu, \vv \in \mathbb{R}^d$: $ \innp{T(\vu) - T(\vv), \vu - \vv} \geq \frac{1}{L} \| T(\vu) - T(\vv) \|^2$.
Monotonicity can be defined in the standard way for a multi-valued operator $T\colon\mathbb{R}^d\rightrightarrows\mathbb{R}^d$. 
Note that $\frac{1}{L}$-cocoercivity  implies monotonicity and $L$-Lipschitzness, but not vice versa. \emph{Maximal monotone} operators are those whose graph is not properly contained in the graph of any other monotone operator. Common examples for this class include subdifferentials of proper convex l.s.c.\ functions. For a further discussion of these properties, see the textbook~\citet{bauschke2011convex}. 

Given an operator $T$, its resolvent operator is defined as $J_{\eta T} = (\mathrm{Id} + \eta T)^{-1}$ for some $\eta > 0$, i.e., 
\begin{align*}
    \Bar{\vu} \in J_{\eta T}(\vu) \iff \frac{1}{\eta}(\vu - \Bar{\vu}) \in T(\Bar{\vu}).
\end{align*}
Important instances of resolvents include the proximal operator obtained when $T = \partial g$ for a convex function $g$ and projection $P_C$ obtained when $T = \partial \delta_C$ for the indicator function $\delta_C$ of a closed convex set $C$.
An important and useful property of the resolvent operator $J_{\eta T}$ is that it is single valued and nonexpansive ($1$-Lipschitz) when $T$ is maximally monotone. 

Our work leverages the classical Halpern iteration \citep{halpern1967fixed} for solving fixed point equations $\vx = T(\vx)$ with nonexpansive operators $T.$ Halpern iteration is defined by
\begin{align}\label{eq:halpern}
    \vu_{k + 1} = \lambda_k\vu_0 + (1 - \lambda_k)T(\vu_k), \tag{Hal}
\end{align}
where $\lambda_k$ is a step parameter typically set to be of the order-$\frac{1}{k}$.
To address~\eqref{eq: prob1}, several variants of~\eqref{eq:halpern} have been proposed, with different choices of the nonexpansive operator $T$. Notable examples most relevant to our work include $T = J_{\eta G}\circ (\mathrm{Id} - \eta F)$ for cocoercive settings and $T = J_{\eta(F + G)}$ for monotone Lipschitz settings~\citep{diakonikolas2020halpern}. We defer other background information about the techniques used in the paper to Appendix~\ref{app: background}, due to space constraints.

We start with the common standard assumption that we require in all of the results.
\begin{assumption}\label{asp: asp1}
    The operator $F\colon\mathbb{R}^d\to\mathbb{R}^d$ is monotone and $L_F$-Lipschitz, and the operator $G\colon\mathbb{R}^d\rightrightarrows\mathbb{R}^d$ is maximally monotone. The solution set of~\eqref{eq: prob1} is nonempty, i.e., $(F+G)^{-1}(\mathbf{0}) \neq \emptyset$.
\end{assumption}
The next two assumptions characterize the two separate settings we consider in the sequel. 
\begin{assumption}\label{asp: asp2}
    The operator $F\colon\mathbb{R}^d\to\mathbb{R}^d$ is $L_Q$-Lipschitz in expectation, meaning that given an oracle $F_\xi$ with distribution $Q$ such that $\mathbb{E}_{\xi\sim Q}[F_\xi(\vu)] = F(\vu)$, we have for any $\vu, \vv \in \sR^d$, 
    \begin{equation*}
    \mathbb{E}_{\xi\sim Q}\| F_\xi(\vu) - F_\xi(\vv) \|^2 \leq L_Q^2 \| \vu - \vv \|^2.
\end{equation*}
\end{assumption}
This assumption holds, for example, when each $F_i$ is Lipschitz-continuous, and is standard for analyzing variance reduced algorithms (see e.g.,~\citet{palaniappan2016stochastic,carmon2019variance,alacaoglu2021stochastic}). 
Similarly, in the finite-sum case, we assume that $F$ is cocoercive on average, which can be regarded as cocoercivity in expectation with uniform sampling.
\begin{assumption}\label{asp: asp3}
    The operator $F\colon\mathbb{R}^d\to\mathbb{R}^d$ is $\frac{1}{L}$-cocoercive on average, i.e.,\ for any $\vu, \vv \in \sR^d$
    \begin{equation*}
    \textstyle
    \langle F(\vu)- F(\vv), \vu - \vv \rangle \geq \frac{1}{nL}\sum_{i=1}^n\|F_i(\vu) - F_i(\vv)\|^2.
\end{equation*}
\end{assumption}
This assumption holds, for example, when each $F_i$ is cocoercive. In the minimization case, this corresponds to the smoothness of component functions, which is standard in variance reduction literature~\citep{allen2017katyusha,nguyen2017sarah}. In the case of fixed point problems, it is implied by the nonexpansiveness of component operators. An example of this case is given as a convex feasibility problem in~\citet[Section 5.2]{Malitsky2019}. Similar assumptions also arise in~\citet{davis2023variance,morin2022cocoercivity,tran2023randomized,loizou2021stochastic}. 

\paragraph{Oracle complexity.} As the standard convention for finite-sum problems~\citep{palaniappan2016stochastic,carmon2019variance,alacaoglu2021stochastic}, we measure the oracle complexity of an algorithm by the number of calls to $F_i$ to make an optimality measure, residual norm in our case, small (the number of calls to the resolvent of $\eta G$ for $\eta > 0$ is of the same order). Since our variance reduced estimators compute the full operator values with some probability, per-iteration costs are random and our complexity results are on the \emph{expected} number of calls to $F_i$. This is also a standard way to measure complexity with single-loop variance reduced algorithms~\citep{li2021page,kovalev2020don}. It is possible to obtain deterministic per iteration costs by exchanging to multi-loop algorithms~\citep{carmon2019variance,alacaoglu2021stochastic}, which we avoid for simplicity.

\section{Algorithm and Analysis in the Cocoercive Case}\label{sec:expected-cocoercive}
In this section, we consider the cocoercive setting where $F$  satisfies Assumption~\ref{asp: asp3}. The main reason that we study this case separately is because we can provide a simpler single-loop algorithm under cocoercivity. We use the following stochastic variant of a constrained Halpern iteration 
\begin{align}
    \vu_{k+1} = J_{\eta G}\big(\lambda_k \vu_0 + (1-\lambda_k)\vu_k - \eta \tildeF(\vu_k)\big),
\end{align}
where $\tildeF$ is the variance-reduced PAGE estimator~\citep{li2021page}. We summarize our approach in Alg.~\ref{alg: cocoercive}, and defer the details of the PAGE estimator to Appendix~\ref{app: background}. 
\begin{algorithm}[h]
\caption{Halpern iteration with variance reduction}
\begin{algorithmic}
    \STATE {\bfseries Input:} $\vu_0 \in \sR^d$, step size
    $\eta = \frac{1}{4L}$, batch size $b = \lceil \sqrt{n} \rceil$ \\
    \STATE $\vu_1 = J_{\frac{\eta}{2\lambda_1}G}\big(\vu_0 - \frac{\eta}{2\lambda_1} F(\vu_0)\big)$, \quad $\tildeF(\vu_1) = F(\vu_1)$ \\
    \vspace{.2cm}
    \FOR{$k = 1, 2,\ldots $}
        \STATE $\lambda_k = \frac{2}{k + 4}$, \quad $p_{k + 1} = \begin{cases}
            \frac{4}{k + 5} \quad &  \forall k \leq \sqrt{n} \\
            \frac{4}{\sqrt{n} + 5} \quad & \forall k \geq \sqrt{n}
        \end{cases}$ \\
        \STATE $\vu_{k+1} = J_{\eta G}(\lambda_k \vu_0 + (1-\lambda_k)\vu_k - \eta \tildeF(\vu_k))$ \\
        \STATE Sample $\gS_{k + 1} \subseteq \{1, \dots, n\}$ without replacement and uniformly such that $|\gS_{k + 1}|=b$
    \STATE $\tildeF(\vu_{k+1}) = \begin{cases}
        F(\vu_{k+1}) \quad & \text{w.p. } p_{k+1},  \\
        \tildeF(\vu_{k}) + \frac{1}{b}\sum_{i \in \gS_{k + 1}}\big(F_i(\vu_{k+1}) - F_i(\vu_{k})\big) \quad & \text{w.p. } 1 - p_{k+1}.
    \end{cases}$
        \ENDFOR
      \end{algorithmic}
\label{alg: cocoercive}
\end{algorithm}

A similar constrained Halpern iteration scheme has been analyzed in~\citet{cai2022accelerated,kovalev2022first} with an extrapolation step, but only for deterministic settings. For the stochastic counterpart, Alg.~\ref{alg: cocoercive} can be seen as a simpler constrained version of~\cite{cai2022stochastic}, with a single parameter for the constant batch size that (unlike in the algorithm of \citet{cai2022stochastic}) is independent of the accuracy $\varepsilon$, the component variance, the norm of iterate differences, and a stage-wise choice of $p_k$. 
The reason we are able to simplify the batch size selection comes from our focus on the finite-sum problems, whereas \citet{cai2022stochastic} considered infinite-sum problems.

The key technical ingredient in our analysis is the following lemma, which shows that, in expectation, in each iteration $k$ we can reduce a potential function by a factor $(1-\lambda_k)$. This potential reduction, in turn, can be translated into the residual error decay at rate $\lambda_k = \mathcal{O}(1/k)$, as shown in Theorem~\ref{th: theorem_cocoercive} below. Due to space constraints, the proofs are deferred to Appendix~\ref{appx:cocoercive}.

\begin{restatable}{lemma}{ckCoco}
\label{lem: cocoercive_decrease}
    Let Assumptions~\ref{asp: asp1} and~\ref{asp: asp3} hold. 
    Then, for the iterates $\vu_k$ of Algorithm~\ref{alg: cocoercive} and the potential function $\gC_k$ defined by 
    \begin{equation}\label{eq: ck_def}
    \gC_k = \frac{\eta}{2\lambda_k}\|F(\vu_k) + \vg_k\|^2 + \innp{F(\vu_k) + \vg_k, \vu_k - \vu_0} + c_k\|F(\vu_k) - \tildeF(\vu_{k})\|^2,
\end{equation} 
we have that $   \mathbb{E}[\gC_{k+1}] \leq (1-\lambda_k) \mathbb{E}[\gC_k]$ for $k\geq 1$, where $\vg_{k + 1} = \frac{1}{\eta}\big(\lambda_k\vu_0 + (1 - \lambda_k)\vu_k - \eta\tildeF(\vu_k) - \vu_{k + 1}\big) \in G(\vu_{k + 1})$ and $c_k = \frac{(\sqrt{n} + 2)(k + 4)}{4L}$.
\end{restatable}
Our new potential function allows us to go beyond the deterministic setting analyzed in~\cite{kovalev2022first}, while handling the variance of the estimator $\tildeF$, which also helps us avoid the more complicated induction-based argument in~\citet{cai2022stochastic}.
 Another important aspect in the above bound is that $c_1$ can be of the order of $\sqrt{n}$. Hence, to make sure that we do not introduce spurious dependence on $n$, it is critical that the first two iterations of the algorithm evaluate the full operator. 
 
The following theorem states our main convergence result for this section.
\begin{restatable}{theorem}{thmCoco}
\label{th: theorem_cocoercive}
Let Assumptions~\ref{asp: asp1} and~\ref{asp: asp3} hold.
Then, for the iterates $\vu_k$ of Algorithm~\ref{alg: cocoercive}, we have  
\begin{align*}
    \E[\mathrm{Res}_{F + G}(\vu_k)] \leq \big(\E[\mathrm{Res}^2_{F + G}(\vu_k)]\big)^{1/2} \leq \frac{16L\|\vu_0 - \vu_*\|}{k + 4}.
\end{align*}
In particular, given accuracy $\varepsilon > 0$, to return a point $\vu_K$ such that $\E[\mathrm{Res}_{F + G}(\vu_K)]
\leq \varepsilon$, the stochastic oracle complexity of Algorithm~\ref{alg: cocoercive} is $\widetilde{\gO}\big(n + \frac{\sqrt{n}L\|\vu_0 - \vu_*\|}{\varepsilon}\big)$.
\end{restatable}
Observe that we use a batch size $|\gS_k| =\lceil \sqrt{n}\rceil$ to obtain our improvement from the employed variance reduction strategy,
which is a  common practice for stochastic algorithms with residual guarantees~\citep{cai2022stochastic,lee2021fast}. Prior work~\citep{pethick2023solving} that avoids a large batch size requires $\gO(\varepsilon^{-4})$ complexity and only provides residual guarantees on the best iterate. We also argue that in the finite-sum case, there is no inherent disadvantage of using $\mathcal{O}(\sqrt{n})$ samples in every iteration since we provably show that this leads to a better dependence on $n$ in the final oracle complexity compared to deterministic algorithms, which would use $n$ samples every iteration.

To compare with prior results, we first note that deterministic Halpern iteration for constrained VIs with cocoercive operators yields $\widetilde{\mathcal{O}}(nL_F\varepsilon^{-1})$ complexity~\citep{diakonikolas2020halpern}, for which our result in Theorem~\ref{th: theorem_cocoercive} replaces $nL_F$ with $\sqrt{n}L$ and can provide improvements up to $\sqrt{n}$ depending on the relationship between $L$ and $L_F$ (see examples in~\citet{palaniappan2016stochastic,carmon2019variance,alacaoglu2021stochastic}). Moreover, compared to complexity results $\mathcal{O}(L_F\varepsilon^{-3})$ and $\widetilde{\mathcal{O}}(L_F\varepsilon^{-2})$ for algorithms developed for the infinite-sum stochastic settings in~\citet{cai2022stochastic,chen2022near}, we provide improvements in the regime $\varepsilon=o({1}/{\sqrt{n}})$, assuming $L_F \approx L$.

An important implication of cocoercive inclusions is for the finite-sum minimization where variance reduction has been studied extensively. The state-of-the-art results with direct algorithms are due to~\cite{lan2019unified} and~\cite{song2020variance}, and they provide oracle complexity $\widetilde{\mathcal{O}}(n+\sqrt{nL\varepsilon^{-1}}\|\vu_0 - \vu_*\|)$ for the objective suboptimality. This result can be translated to the norm of the prox-mapping to get $\mathbb{E}[\mathrm{Res}_{F+G}(\vu_{\text{out}})]\leq\varepsilon$ 
with complexity $\widetilde{\mathcal{O}}(n+\sqrt{n}L\|\vu_0-\vu_*\|\varepsilon^{-1})$ which is the same as our result when specified to the case $F=\nabla f$ for a smooth convex function $f$ and $G=\partial g$ for a regularizer $g$. Our results seem to provide the best-known guarantees (among direct approaches) with a single-loop algorithm for the first time. Single-loop versions of Katyusha of~\cite{allen2017katyusha} were studied in~\cite{kovalev2020don} and~\cite{qian2021svrg}, albeit without guarantees for the non-strongly convex case. It is worth noting that this complexity is not optimal for smooth convex minimization and has been improved for unconstrained minimization or with indirect algorithms~\citep{zhou2022practical,allen2018make}.

\section{Algorithm and Analysis in the Monotone Lipschitz Case}\label{sec:lipschitz}
In this section, we consider the more standard setting where $F$ is monotone and $L_Q$-expected Lipschitz for an oracle distribution $Q$. We note that our results apply for general sampling distributions $Q$ under which Assumption \ref{asp: asp2} holds; for concrete examples of beneficial non-uniform sampling distributions, see Remark~\ref{rmk:matix-game}. 
We omit the subscript and denote $L = L_Q$ for brevity in this section, since the context is clear. To obtain the desired  complexity, we make use of the resolvent operator $J_{\eta(F + G)}(\vu)$ for some fixed $\eta > 0$ (specified later in this section). In particular, we adapt the stochastic Halpern iteration to  the following single-valued and cocoercive operator 
\begin{equation*}
    P^\eta(\vu) := \vu - J_{\eta(F + G)}(\vu).
\end{equation*} 
Indeed, for any $\eta > 0$, finding a point $\vu$ such that $\E[\|P^\eta(\vu)\|] \leq \eta \varepsilon$ is sufficient to approximate~\eqref{eq: prob1} of $F + G$, as summarized in the following proposition with the proof deferred to Appendix~\ref{appx:lipschitz}. 
\begin{restatable}{proposition}{pcTores}
\label{prop: pc_to_res}
For any fixed $\eta > 0$, let $P^\eta(\vu) = \vu - J_{\eta(F + G)}(\vu)$. If $\|P^\eta(\vu)\| \leq \eta\varepsilon$ for some $\varepsilon > 0$, then we have $\mathrm{Res}_{F + G}(\Bar{\vu}) \leq \varepsilon$ with $\Bar{\vu} = \vu - P^\eta(\vu) = J_{\eta(F + G)}(\vu)$.
\end{restatable}
This standard proposition gives us a simple way to convert the guarantees on $\|P^\eta(\vu)\|$ to the residual norm~\eqref{eq:res} conceptually, and we later provide a computable approximation for $\bar \vu$. If the resolvent and thus $P^\eta$ can be computed exactly, then~\eqref{eq: prob1} can be solved by the standard, deterministic Halpern iteration, as $P^\eta(\vu)$ is nonexpansive. 
However, the exact evaluation of resolvent operators only happens in special cases, and even for those cases, the computation is usually prohibitive when $n$ is large for the operator $F$. Instead, 
one can efficiently approximate the resolvent by solving a finite-sum strongly monotone VI problem, for which we provide more details in Section~\ref{sec:resolvent}. 

{In the rest of the section, we provide an overview of the analysis and main technical results. Due to space constraints, the proofs are deferred to Appendix \ref{appx:lipschitz}.}

\subsection{Inexact Halpern Iteration with Stochastic Error}\label{sec:inexact-halpern}
Denoting the resolvent approximation by $\widetilde{J}_{\eta(F + G)}$, we use the inexact Halpern iteration as follows 
\begin{equation}\label{eq:inexact-halpern}
\begin{aligned}
\vu_{k + 1} = \lambda_k\vu_0 + (1 - \lambda_k)\widetilde {J}_{\eta(F + G)}(\vu_k) = \lambda_k\vu_0 + (1 - \lambda_k)(\vu_k - P^\eta(\vu_k)) - (1 - \lambda_k)\ve_k, 
\end{aligned}
\end{equation}
where $\ve_k = J_{\eta(F + G)}(\vu_k) - \widetilde{J}_{\eta(F + G)}(\vu_k)$ is the approximation error. To efficiently compute $\widetilde{J}_{\eta(F + G)}$ to a certain accuracy, we use the variance-reduced forward-reflected-backward method ($\mathtt{VR-FoRB}$, Alg.~\ref{alg: forb_vr_sc}) proposed in~\citet{alacaoglu2021stochastic} as our subsolver for each iteration. We summarize our approach in Alg.~\ref{alg:monotone}, and defer our detailed discussion of $\mathtt{VR-FoRB}$ to Section~\ref{sec:resolvent}.

\begin{algorithm}[h]
\caption{Inexact Halpern iteration with $\mathtt{VR{-}FoRB}$}
\begin{algorithmic}
    \STATE {\bfseries Input:} $\vu_0\in\mathbb{R}^d$, $L=L_Q$ with the distribution $Q=\{q_i\}_{i=1}^n$, $n$, $\eta$ \\
    \vspace{.2cm}
    \FOR{$k = 0, 1, 2,\ldots $}
        \STATE $\lambda_k = \frac{1}{k + 2}$, $M_k = \big\lceil 56\max\big\{n, \sqrt{n}(\eta L + 1)\big\}\log\big(1.252(k + 2)\big)\big\rceil$ \\
        \STATE $\widetilde{J}_{\eta(F + G)}(\vu_k) = \mathtt{VR{-}FoRB}(\vu_k, M_k, \text{Id}+\eta(F+G)-\vu_k, Q)$ \\
        \STATE $\vu_{k + 1} = \lambda_k\vu_0 + (1 - \lambda_k)\widetilde J_{\eta(F + G)}(\vu_k)$
        \ENDFOR
      \end{algorithmic}
\label{alg:monotone}
\end{algorithm}
Halpern iteration with inexact resolvent computation has been shown to maintain $\gO(1/k)$ convergence rate in \emph{deterministic} settings~\citep{diakonikolas2020halpern}, provided the approximation error $\|\ve_k\|$ is sufficiently small. 
The critical difference is that we can no longer use the stopping criterion for the inner algorithm therein, due to the randomized nature of $\mathtt{VR{-}FoRB}$. We also observe that their inexactness criterion $\|\ve_k\| \leq \frac{\varepsilon}{4k(k+1)}$ for each iteration $k$ requires a pre-fixed accuracy $\varepsilon$ and also usually leads to the bound on the number of inner iterations dependent on $J_{\eta(F + G)}(\vu_k)$ which is not feasible empirically. The latter is simply because the initial distance to the solution of the subproblem appears in the complexity bound (see, e.g.,\ Theorem~\ref{thm:strongly-monotone}). To overcome these issues, we use a different inexactness criterion that $\E_k[\|\ve_k\|^2] \leq \frac{\|P^\eta(\vu_k)\|}{(k + 2)^8}$, conditional on the algorithm randomness up to iteration $k$. Such a criterion can be guaranteed to hold by setting the number of inner iterations to be sufficiently high yet only dependent on known constants, using the convergence results of $\mathtt{VR{-}FoRB}$ from Section~\ref{sec:resolvent}. We summarize these results in the following theorem. 

\begin{restatable}{theorem}{thmMono}
\label{thm:monotone}
    Let Assumptions~\ref{asp: asp1}~and~\ref{asp: asp2} hold. 
    Then, for the iterates $\vu_k$ of Algorithm~\ref{alg:monotone}, we have $\E_k[\|\ve_k\|^2] \leq \frac{\|P^\eta(\vu_k)\|}{(k + 2)^8}$ conditional on the algorithm randomness up to iteration $k$, and 
\begin{align*}
    \E[\|P^\eta(\vu_k)\|] \leq \big(\E[\|P^\eta(\vu_k)\|^2]\big)^{1/2} \leq \frac{7L\|\vu_0 - \vu_*\|}{k}.
\end{align*}
Moreover, given accuracy $\varepsilon > 0$, to return a point $\vu_K$ such that $\E[\|P^\eta(\vu_K)\|]
\leq \eta\varepsilon$ with $\eta = \frac{\sqrt{n}}{L}$, the stochastic oracle complexity is $\Tilde{\gO}\big(n + \frac{\sqrt{n}L\|\vu_0 - \vu_*\|}{\varepsilon}\big)$.
\end{restatable}
The final step is to characterize the precise point at which we have the residual guarantees. 
\begin{restatable}{corollary}{postprocess}
\label{cor: postprocess}
    Let Assumptions~\ref{asp: asp1} and~\ref{asp: asp2} hold and let $\vu_K$ be as defined in Theorem~\ref{thm:monotone}.
    Then, for $\vu_{\mathrm{out}}=\mathtt{VR{-}FoRB}(\vu_K, \lceil 42(n+\sqrt{n})\log (19n) \rceil, \textup{Id} + \eta(F + G) - \vu_K, Q)$ with $\eta = \frac{\sqrt{n}}{L}$,
    \begin{equation}\notag
        \mathbb{E}[\mathrm{Res}_{F+G}(\vu_{\mathrm{out}})] \leq2\varepsilon.
    \end{equation}
     The total stochastic oracle complexity for producing $\vu_{\mathrm{out}}$ is $\widetilde{\mathcal{O}}\left(n+\frac{\sqrt{n}L\|\vu_0 - \vu_*\|}{\varepsilon} \right)$.
\end{restatable}
\begin{remark}\label{rmk:matix-game}
Non-uniform sampling $Q$ can be beneficial in terms of lowering the Lipschitz constant $L_Q$ and thus the overall algorithm complexity. As a specific example, consider the matrix game 
\begin{equation}\notag 
\min_{\vx \in \sR^{m_1}} \max_{\vy \in \sR^{m_2}} \innp{\mA\vx, \vy} + \delta_{\Delta^{m_1}}(\vx) + \delta_{\Delta^{m_2}}(\vy)
\end{equation}
for $\mA \in \mathbb{R}^{m_1 \times m_2}$, the simplices $\Delta^{m_1}$, $\Delta^{m_2}$, where $\delta$ is the indicator function. With $\vu=\binom{\vx}{\vy} \in \sR^{m_1 + m_2}$, we have $F(\vu)=\binom{\mA^\top \vy}{-\mA \vx}$ and $G(\vu) = \binom{\partial \delta_{\Delta^{m_1}}(\vx)}{\partial \delta_{\Delta^{m_2}}(\vy)}$. Denote the $i$-th row and the $j$-th column of $\mA$ by $\mA_{i:}$ and $\mA_{:j}$, respectively. Let $\|\cdot\|_2$ be the operator norm for a matrix, and $\|\cdot\|_F$ be its Frobenius norm. 
Consider the standard importance sampling for $Q$, i.e., we sample $\xi=(i, j)\sim Q$ such that 
\begin{equation*}
    F_\xi(\vu) = \binom{\frac{1}{q_i^{(1)}}\mA_{i:}\vy_i}{-\frac{1}{q_j^{(2)}}\mA_{:j}\vx_j}, ~~~ \mathbb{P}_Q[\xi = (i, j)] = q_i^{(1)}q_j^{(2)},~~q_i^{(1)} = \frac{\|\mA_{i:}\|^2_2}{\|\mA\|^2_F}, ~~~ q_j^{(2)} = \frac{\|\mA_{:j}\|^2_2}{\|\mA\|^2_F}.
\end{equation*}
It is easy to verify that $\E_{\xi \sim Q}[F_\xi(\vu)] = F(\vu)$, and we have $L_Q = \|A\|_F$ while $L_F=\|A\|_2$. Since it is possible for $\|A\|_F$ and $\|A\|_2$ to be of the same order, in those cases the improvement from the variance reduction approaches (including ours) is of the order $\sqrt{\frac{2m_1 m_2}{m_1 + m_2}}$ (order $\sqrt{m_1}$ for square matrices). Similar conclusions can be drawn more generally for linearly constrained nonsmooth convex optimization problems; see \cite[Section 4]{alacaoglu2021stochastic} and App.~\ref{subsec: rmk} for details.
\end{remark}

{ In addition to complexity bounds for the expected residual, our results also have a direct consequence for high probability guarantees. In particular, since our result clearly implies  $\mathrm{Res}_{F+G}(\vu_{\text{out}})\leq\varepsilon$ with a constant probability by Markov's inequality and since the residual is computable, one can run the algorithm multiple times and monitor the residual, to obtain a high probability guarantee with logarithmic dependence on the confidence level. See, for example,~\cite{zhou2022practical, allen2018make} where such a confidence boosting mechanism is used in a similar context, for minimization.}

A few other remarks are in order here. First, our results imply the gap guarantees in prior work~\citep{alacaoglu2021stochastic,carmon2019variance}, while the algorithms therein are bound to be suboptimal for the residual since they reduce to the exragradient algorithm in the case $n=1$, which is known to be suboptimal for the residual guarantee~\citep{golowich2020last}. Moreover, residual guarantees for these variance reduced algorithms are currently unknown. The implication to gap guarantees also ensures the near-optimality of our results since such complexity is known to be unimprovable for the gap guarantees~\citep {alacaoglu2021stochastic,han2021lower}. 

Next, compared to  deterministic algorithms with $\Tilde{\mathcal{O}}( nL_F\varepsilon^{-1} )$ complexity for the residual~\citep{diakonikolas2020halpern,Yoon2021OptimalGradientNorm}, we replace $nL_F$ with $\sqrt{n}L$. This brings improvements in important cases discussed in~\citet{palaniappan2016stochastic,carmon2019variance,alacaoglu2021stochastic}, including linearly constrained problems and matrix games discussed in Remark \ref{rmk:matix-game}. 
This mirrors the recent improvements for the duality gap guarantees for VIs where~\citet{alacaoglu2021stochastic,carmon2019variance} showed $O(n+\sqrt{n}L\varepsilon^{-1})$ complexity instead of $\mathcal{O}(nL_F\varepsilon^{-1})$ of deterministic methods~\citep{nemirovski2004prox}. 

Finally, we show the extension to the cohypomonotone settings (see \cite[Def. 2.2]{combettes2004proximal}, \cite[Def. 2.3(iii)]{bauschke2021generalized}) defined by the existence of $\rho>0$ such that $\langle F(\vu) - F(\vv), \vu-\vv \rangle \geq -\rho\|F(\vu)-F(\vv)\|^2$ in the following corollary for completeness, with more justfications provided in Appendix~\ref{subsec: postprocess}. { This result provides a better dependence on $n$ compared to previous results with the drawback of a more restrictive bound for $\rho$ (roughly, $\rho< \frac{L}{\sqrt{n}L_F^2}$) and using $G\equiv 0$. See also \citet{kohlenbach2022proximal} who used a similar idea for solving cohypomonotone problems without explicit complexity analysis in terms of first-order oracles.}
\begin{restatable}{corollary}{postproc}
\label{cor: comonotone}[Cohypomonotone]
Assume that $F$ is maximally $\rho$-cohypomonotone and $L$-Lipschitz in expectation and $G\equiv 0$. Then, given $\varepsilon > 0$, Alg.~\ref{alg:monotone} returns a point $\vu_K$ such that $\big(\E[\|P^\eta(\vu_K)\|^2]\big)^{1/2} \leq \eta\varepsilon$ with $\widetilde{O}\Big(\big(n + \sqrt{n}\frac{\eta L + 1}{1-\rho \eta L_F^2}\big)\big(\frac{\|\vu_0 - \vu_*\|}{\eta\varepsilon} + 1\big)\Big)$ stochastic oracle complexity, for any positive $\eta$ such that $\rho < \min\left(\frac{\eta}{2}, \frac{1}{\eta L_F^2}\right)$. With $\eta=\frac{\sqrt{n}}{L}$ as before, this corresponds to $\rho< \min\left(\frac{L}{\sqrt{n}L_F^2}, \frac{\sqrt{n}}{2L}\right)$ and complexity $\widetilde{O}\left(\left(n + \frac{\sqrt{n}L\|\vu_0-\vu_\star\|}{\varepsilon} \right)\frac{1}{1-\frac{\rho \sqrt{n}L_F^2}{L}}\right)$.
\end{restatable}

\subsection{Randomized Approximation of the Resolvent}\label{sec:resolvent}
Approximating $J_{\eta(F + G)}(\vu^{+})$ for a given $\vu^{+} \in \sR^d$ corresponds to solving the finite-sum MI defined by 
\begin{align}\label{prob:SVI}
    \text{find $\Bar{\vu}$ such that } \quad \mathbf{0} \in 
    \eta F(\Bar{\vu}) + \eta G(\Bar{\vu}) + \Bar{\vu} - \vu^+. 
\end{align}
It is immediate that the solution to~\eqref{prob:SVI} of the operator $\eta (F + G) + \mathrm{Id} - \vu^+$
corresponds to $J_{\eta(F + G)}(\vu^+)$ by the definition of the resolvent. Note that $\eta (F + G) + \mathrm{Id} - \vu^+$ can be represented as a sum of two operators $\Bar{F}^\eta(\vu; \vu^+)$ and $\eta G(\vu)$, where 
\begin{align}\label{eq:barF}
    \Bar{F}^\eta(\vu; \vu^+) := \eta F(\vu) + \vu - \vu^+ = \frac{1}{n}\sum_{i = 1}^n\Bar{F}^\eta_i(\vu; \vu^+), \quad \Bar{F}^\eta_i(\vu; \vu^+) := \eta F_i(\vu) + \vu - \vu^+.
\end{align}
It is simple to verify that $\Bar{F}^\eta(\vu; \vu^+)$ is $1$-strongly monotone and $(\eta L + 1)$-average Lipschitz w.r.t.\ $\vu$; we provide a proof in Appendix~\ref{appx:resolvent} for completeness, and $\eta G$ is still maximally monotone as $\eta > 0$. Hence below we use a more general notation and will set $A(\vu) = \bar F^\eta(\vu; \vu^+)$ and $B=\eta G$. As mentioned before, we use $\mathtt{VR{-}FoRB}$ from~\cite[Algorithm 4]{alacaoglu2021stochastic} to solve~\eqref{prob:SVI}, which we present as Alg.~\ref{alg: forb_vr_sc}. We now state its convergence result under strong monotonicity. 
\begin{algorithm}[h]
\caption{$\mathtt{VR{-}FoRB}(\vu, M, A+B, Q)$~\cite[Algorithm 4]{alacaoglu2021stochastic}}
\begin{algorithmic}
    \STATE {\bfseries Input:} $\vv_0=\vw_0 = \vw_{-1} = \vu$, $p = \frac{1}{n}$, $\alpha = 1 - p$, $\tau = \frac{\sqrt{p(1-p)}}{2L_A}$, distribution $Q=\{q_i\}_{i=1}^n$
    \vspace{.2cm}
    \FOR{$k = 0, 1, \ldots, M - 1$}
        \STATE $\hat \vv_k = \alpha \vv_k + (1-\alpha) \vw_k$ \\
        \STATE Sample $i\in\{1,\dots,n\}$ according to $Q$ \\
        \STATE $\vv_{k+1} = J_{\tau B}(\hat \vv_k - \tau[A(\vw_k) - (nq_i)^{-1}A_i(\vw_{k-1}) + (nq_i)^{-1}A_i(\vv_k)])$\\[1mm]
        \STATE $\vw_{k+1} = \begin{cases}\vv_{k+1} \quad & \text{~w.p.~} p\\
    \vw_k \quad & \text{~w.p.~} 1-p\end{cases}$
        \ENDFOR
      \end{algorithmic}
\label{alg: forb_vr_sc}
\end{algorithm}
\begin{restatable}{theorem}{FoRB}
\label{thm:strongly-monotone}
    Let $A \colon\mathbb{R}^d\to\mathbb{R}^d$ be monotone and $L_A$-Lipschitz in expectation with $A=\frac{1}{n}\sum_{i=1}^n A_i$, $B \colon \mathbb{R}^d \rightrightarrows \mathbb{R}^d$ be maximally monotone, and $A+B$ be $\mu$-strongly monotone with $(A+B)^{-1}(\mathbf{0}) \neq \emptyset$ and $\vv_* = (A+B)^{-1}(\mathbf{0})$. 
    Given $\Bar{\varepsilon} > 0$, Alg.~\ref{alg: forb_vr_sc} returns $\vv_M$ with $\E[\| \vv_{M} - \vv_*\|^2] \leq \Bar{\varepsilon}^2$ 
    in $\big\lceil 14\max\{n, \frac{\sqrt{n}L_A}{\mu}\}\log(\frac{\sqrt{6}\|\vv_0 - \vv_*\|}{\bar \varepsilon})\big\rceil$ iterations and $\gO\big( ( n + \frac{\sqrt{n}L_A}{\mu}) \log\big(\frac{\|\vv_0 - \vv_*\|}{\Bar{\varepsilon}}\big) \big)$ oracle queries.
\end{restatable}
We remark that only almost sure convergence was proved for $\mathtt{VR{-}FoRB}$ in~\citet{alacaoglu2021stochastic}. We show its non-asymptotic linear convergence  which is needed for our main result in Theorem~\ref{thm:monotone}. A similar rate is in~\cite[Theorem 6]{alacaoglu2021stochastic} for strongly monotone VIs, but for a different algorithm -- variance reduced extragradient.
Our result can be seen as a ``single-call'' alternative to this method, for the slightly more general setting of strongly monotone inclusions. We derive it to keep the generality of our setting and also for possible interest in its own right, since such algorithms have been popular recently~\citep{cai2022acceleratedsingle,hsieh2019convergence}.

\section{Numerical Experiments}
We provide preliminary numerical results for our proposed algorithms. We compare Alg.~\ref{alg: cocoercive} and Alg.~\ref{alg:monotone} with existing algorithms on matrix games and a special quadratic program used for lower bound derivations in~\cite{ouyang2021lower}. We emphasize that our main contributions are theoretical, while the provided examples are mainly for illustration with two goals in mind: (\emph{i}) verify our improved complexity bounds compared to deterministic Halpern-based methods~\citep{cai2022accelerated}, (\emph{ii}) show benefits of our algorithms compared to prior variance reduced algorithms~\citep{alacaoglu2021stochastic} for \emph{difficult} problem instances used for establishing lower bounds.

We compare our algorithms with deterministic extragradient (EG)~\citep{korpelevich1977extragradient}, constrained extra anchored gradient method (EAG)~\citep{cai2022accelerated}, and variance-reduced extragradient (VR-EG)~\citep[Alg.~1]{alacaoglu2021stochastic}.
First problem we consider is a matrix game with simplex constraints, i.e., $\min_{\vx \in \Delta^{m_1}} \max_{\vy \in \Delta^{m_2}} \innp{\mA\vx, \vy}$ with $m_1 = m_2 = 500$. We use the policeman and burglar matrix~\citep{nemirovski2013mini}.
Second problem we consider is a quadratic program from~\citet{ouyang2021lower} equivalent to the problem $\min_{\vx \in \R^{m_1}}\max_{\vy \in \R^{m_2}} \frac{1}{2}\vx^\top \mH \vx - \vh^\top\vx - \innp{\mA\vx - \vb, \vy}$. This problem was used in~\citet{ouyang2021lower} for establishing lower bounds for min-max optimization and we use this example with $m_1 = m_2 = 200$ to show the efficacy of Halpern-type algorithms. We use uniform sampling for all the algorithms, set $M_k = \lfloor 0.05n\log(k + 2)\rfloor$ for Alg.~\ref{alg:monotone}, and tune the stepsizes for each method individually. We implement all the algorithms in Python, and run the experiments on Google Colab standard CPU backend. We provide further experiment details in Appendix~\ref{appx:exp}.

We summarize our numerical results and plot the operator norm against the number of epochs in Fig.~\ref{fig:num-exp}, where one epoch means $n$ individual component operator evaluations. Operator norm stands for $\|F(\vu)\|$ for the unconstrained case, and we follow the convention and use the norm of gradient mapping, i.e., $\sqrt{\|\vx - \gP_{\Delta^{m_1}}(\vx - \mA^\top\vy)\|^2 + \|\vy - \gP_{\Delta^{m_2}}(\vy + \mA\vx)\|^2}$ for the matrix game (which our guarantees can be directly translated to, see for example~\cite[Fact 1]{cai2022accelerated}). We observe that (\emph{i}) our variance reduced Alg.~\ref{alg: cocoercive} and Alg.~\ref{alg:monotone} converge faster than deterministic methods in both cases, validating our complexity results; (\emph{ii}) VR-EG exhibits slightly faster convergence than our Halpern-type algorithms in Fig.~\ref{fig:matrix-game} (see similar empirical observations and comments in~\citet{park2022exact,tran2023sublinear}), however  VR-EG suffers stagnation while our algorithms progress on the difficult worst-case problem, as shown in Fig.~\ref{fig:lagrangian}.

\begin{figure}[ht]
    \hspace*{\fill}\subfigure[Matrix game]{\includegraphics[width=0.4\textwidth]{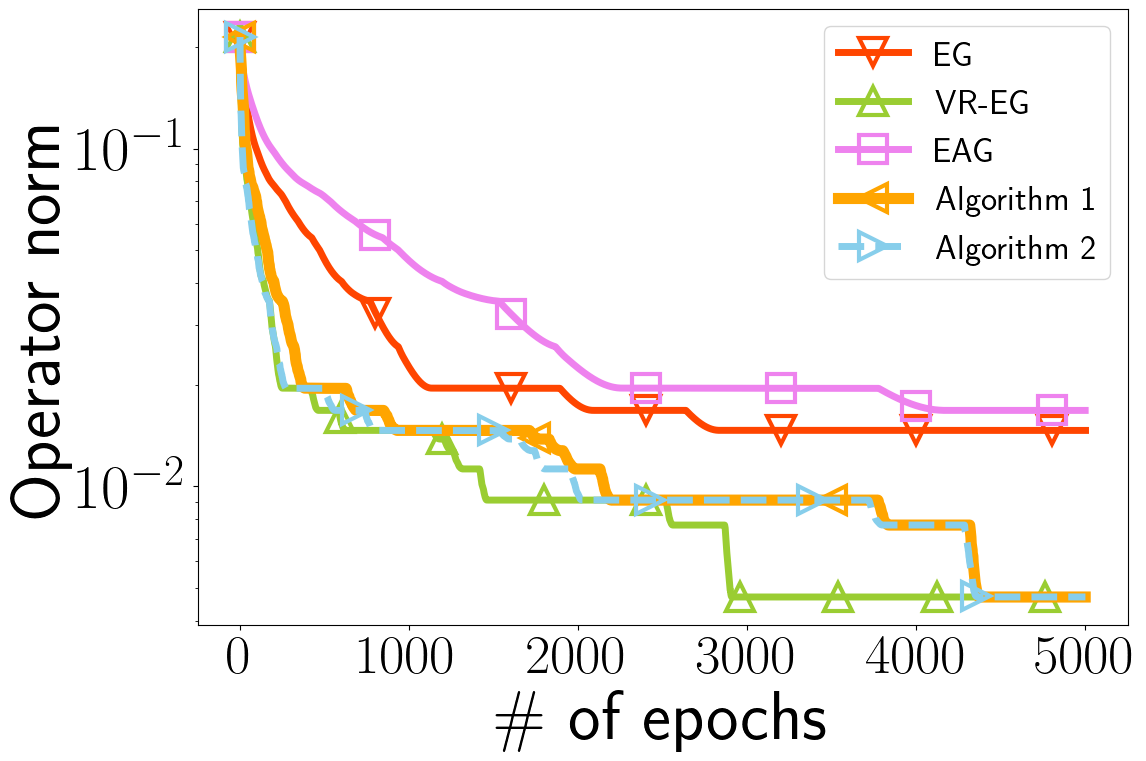}\label{fig:matrix-game}}\hfill
    \subfigure[Quadratic program]{\includegraphics[width=0.4\textwidth]{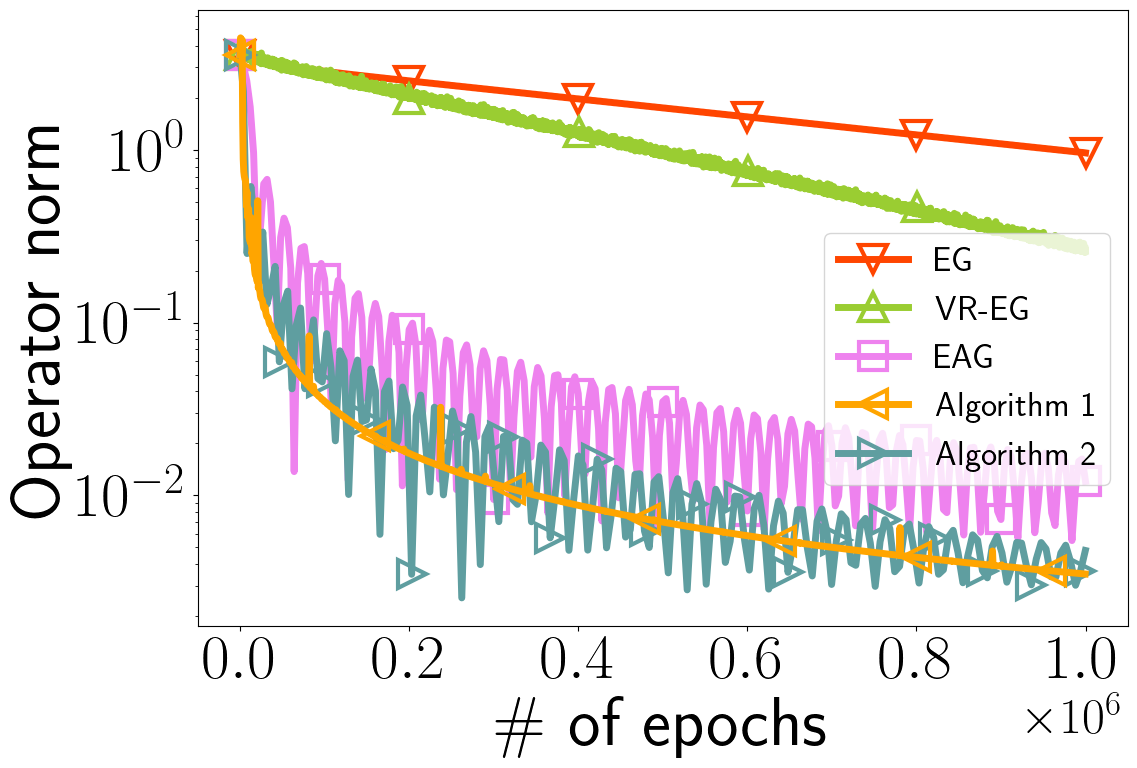}\label{fig:lagrangian}}
    \hspace*{\fill}
    \caption{Comparison of Alg.~\ref{alg: cocoercive}, Alg.~\ref{alg:monotone}, deterministic extragradient (EG), extra anchored gradient (EAG), variance reduced extragradient (VR-EG)  on the matrix games and quadratic programming.}
    \label{fig:num-exp}
\end{figure}

\section{Conclusion}\label{sec: conclusion}
We showed complexity guarantees for variance reduced algorithms that improve the best-known results for minimizing the residual in finite-sum monotone inclusions. Our improvements mirror those that were shown for the duality gap for finite-sum VIs in the recent literature; see e.g.,~\citet{palaniappan2016stochastic,carmon2019variance,alacaoglu2021stochastic}. 

Our result for the cocoercive case is with a direct algorithm whereas for the Lipschitz monotone case, we have an indirect algorithm. In particular, our algorithm in the latter case works by solving randomized approximations to the resolvent, which can be seen as an inexact Halpern iteration~\citep{diakonikolas2020halpern}. An important related open question
is the development of direct algorithms that achieve the same complexity guarantee that we showed for  the Lipschitz monotone case, which is optimal up to log factors. We provide more perspectives about this question in Appendix~\ref{app: perspectives}.

\section*{Acknowledgements}
This research was supported in part by the NSF grant 2023239, the NSF grant 2007757, the NSF grant 2224213, the DOE ASCR Subcontract 8F-30039 from Argonne National Laboratory, the Office of Naval Research under contract number N00014-22-1-2348. 

\bibliography{arxiv_ref}
\bibliographystyle{iclr2024_conference}

\newpage
\appendix
\section{Background}\label{app: background}
\paragraph{VI algorithms, extragradient, Forward-Backward, Proximal Point. }
Two of the most fundamental algorithms for solving VI problems with monotone operators are Forward-Backward (FB) and proximal-point (PP) algorithms, see for example~\citep{facchinei2007finite,rockafellar1976monotone}. FB iterates as
\begin{equation*}
    \vu_{k+1} = J_{\tau G}(\vu_k -  \tau F(\vu_k)),
\end{equation*}
and converges when $F$ is cocoercive.
On the other hand, PP iterates as
\begin{equation*}
    \vu_{k+1} = J_{\tau(F+G)}(\vu_k),
\end{equation*}
and does not require cocoercivity. However, the computation of $J_{\tau (F+G)}$ is in general nontrivial even when $J_{\tau G}$ can be computed efficiently. Hence, the advantage of not requiring cocoercivity comes at the cost of a more expensive iteration.

Extragradient (EG)~\citep{korpelevich1977extragradient} is a classical algorithm that gets the best of both worlds and iterates as
\begin{equation*}
\begin{aligned}
    \vu_{k+1/2} &= P_C(\vu_k - \tau F(\vu_k)) \\
    \vu_{k+1} &= P_C(\vu_k - \tau F(\vu_{k+1/2})).
\end{aligned}
\end{equation*}
This method converges with $L$-Lipschitz operator $F$ and only uses $J_{\tau G}$ for $G = \partial \delta_C$. It turns out that PP and EG have the optimal rates for the gap whereas they are suboptimal for guarantees on the residual, see~\cite{golowich2020last}.

\paragraph{Variance reduction. }
The main idea of variance reduction is to use an estimator $\widetilde F$ such that $\mathbb{E} \| \widetilde F(\vu) - F(\vu)\|^2$ gets progressively smaller as we run the algorithm.
There are several different estimators that are used in minimization such as SVRG~\citep{johnson2013accelerating}, SARAH/SPIDER/PAGE~\citep{nguyen2017sarah,fang2018spider,li2021page} or SAGA~\citep{defazio2014saga,schmidt2017minimizing}. For reference, SVRG estimator, written for the case of operators uses 
\begin{equation*}
    \widetilde F(\vu_k) = F(\vw_k) - F_i(\vw_k) + F_i(\vu_k),
\end{equation*}
for a randomly selected index $i$ and a suitably selected ``snapshot'' point $\vw_k$. A common choice is to select $\vw_k$ to be updated only once every couple of epochs. By using this estimator, recent work~\citet{alacaoglu2021stochastic} showed how to obtain variance reduced extragradient algorithms with optimal complexity for the gap in terms of both the number of operators $n$ and the desired accuracy $\varepsilon$.

Another estimator that is popular for minimization problems is the PAGE estimator~\citep{li2021page}, written for operators as:
\begin{equation}\label{eq:PAGE}
    \widetilde F(\vu_k) = \begin{cases} F(\vu_k), \quad & \text{~w.p.~} p_k \\
    \widetilde F(\vu_{k-1}) + \frac{1}{b}\sum_{i\in \gS_k} \big(F_i(\vu_k) - F_i(\vu_{k-1})\big), \quad & \text{~w.p.~} 1-p_k,\end{cases}
\end{equation}
where $p_k$ and the mini-batch $\gS_k$ with $b = |\gS_k|$ are the parameters.
Even though this estimator has been shown to have unique benefits for minimization, it has not find much use for operators for finite-sum case. It was recently used by~\citet{cai2022stochastic} for operators given as an expectation. Here we introduce a useful result on the variance bound of the PAGE estimator, for our later analysis. Note that this lemma is a slight modification of~\cite[Lemma 3]{li2021page} 
by using without replacement sampling.
\begin{lemma}\label{lem:PAGE}
Let the minibatch $\gS_k$ be uniformly sampled from $[n]$ without replacement. Then the variance of $\tildeF$ defined by Eq.~\eqref{eq:PAGE} satisfies the following recursive bound: for all $k \geq 1,$ 
\begin{align}\notag
\E[\|{\tildeF(\vu_{k}) - F(\vu_{k})}\|^2] &\leq  (1 - p_k)\E[\|{\tildeF(\vu_{k- 1}) - F(\vu_{k - 1})}\|^2] \notag \\
&\quad+ \frac{(n - b)(1-p_k)}{b(n - 1)}\mathbb{E}\left[\sum_{i=1}^n\frac{1}{n}\|F_i(\vu_{k}) - F_i(\vu_{k - 1})\|^2\right].
\end{align}
\end{lemma}
\begin{proof}
Let $\gF_{k}$ denote the filtration that contains all algorithm randomness up to and including iteration $\vu_k$. Using the definition of $\tildeF$ where $\tildeF(\vu_k) = F(\vu_k)$ with probability $p_k$, then conditional on $\gF_k$, we have for all $k \geq 1$ 
\begin{equation*}
\begin{aligned}
    &\E\Big[\|\tildeF(\vu_{k}) - F(\vu_{k})\|^2 ~|~ \gF_{k} \Big] \\
    = \;& (1 - p_k)\E\Big[\Big\|\tildeF(\vu_{k - 1}) + \frac{1}{b} \sum_{i \in \gS_k}\big(F_i(\vu_{k}) - F_i(\vu_{k - 1})\big) - F(\vu_{k})\Big\|^2~\Big|~\gF_k\Big] \\
    = \;& (1 - p_k)\E_{\gS_k}\Big[\Big\|\tildeF(\vu_{k - 1}) + \frac{1}{b} \sum_{i \in \gS_k}\big(F_i(\vu_{k}) - F_i(\vu_{k - 1})\big) - F(\vu_{k})\Big\|^2\Big],
\end{aligned}
\end{equation*}
where $\mathbb{E}_{\gS_k}$ denotes the expectation w.r.t.\ the randomness of $\gS_k$.
Adding and substracting $F(\vu_{k - 1})$ in the quadratic, and noticing that $\E_{\gS_k}\Big[\frac{1}{b}\sum_{i \in \gS_k}\big(F_i(\vu_{k}) - F_i(\vu_{k - 1})\Big] = F(\vu_k) - F(\vu_{k - 1})$, we have 
\begin{align*}
    \;& \E_{\gS_k}\Big[\Big\|\tildeF(\vu_{k - 1}) + \frac{1}{b} \sum_{i \in \gS_k}\big(F_i(\vu_{k}) - F_i(\vu_{k - 1})\big) - F(\vu_{k})\Big\|^2\Big] \\
    = \;& \E_{\gS_k}\Big[\Big\|\frac{1}{b} \sum_{i \in \gS_k}\big(F_i(\vu_{k}) - F_i(\vu_{k - 1})\big) - \big(F(\vu_{k}) - F(\vu_{k - 1})\big)\Big\|^2\Big] + \|F(\vu_{k - 1}) - \tildeF(\vu_{k - 1})\|^2 \\
    & + 2\innp{\tildeF(\vu_{k - 1}) - F(\vu_{k - 1}), \E_{\gS_k}\Big[\frac{1}{b}\sum_{i \in \gS_k}\big(F_i(\vu_{k}) - F_i(\vu_{k - 1}) - \big(F(\vu_k) - F(\vu_{k - 1})\big)\Big]} \\
    = \;& \E_{\gS_k}\Big[\Big\|\frac{1}{b} \sum_{i \in \gS_k}\big(F_i(\vu_{k}) - F_i(\vu_{k - 1})\big) - \big(F(\vu_{k}) - F(\vu_{k - 1})\big)\Big\|^2\Big] + \|F(\vu_{k - 1}) - \tildeF(\vu_{k - 1})\|^2. 
\end{align*}
With $\gS_k$ sampled without replacement according to the uniform distribution, we have (see for example~\cite[Section 2.7]{lohr2021sampling})
\begin{align*}
    \;& \E_{\gS_k}\Big[\Big\|\frac{1}{b} \sum_{i \in \gS_k}\big(F_i(\vu_{k}) - F_i(\vu_{k - 1})\big) - \big(F(\vu_{k}) - F(\vu_{k - 1})\big)\Big\|^2\Big] \\
    \leq \;& \frac{(n - b)}{b(n - 1)}\sum_{i=1}^n \frac{1}{n}\|F_i(\vu_k) - F_i(\vu_{k - 1}) - \big(F(\vu_{k}) - F(\vu_{k - 1})\big)\|^2 \\
    \overset{(\romannumeral1)}{\leq} \;& \frac{(n - b)}{b(n - 1)}\sum_{i=1}^n\frac{1}{n}\|F_i(\vu_k) - F_i(\vu_{k - 1})\|^2,
\end{align*}
where $(\romannumeral1)$ holds because $\E[\|X - \E[X]\|^2] \leq \E[\|X\|^2]$ for any random variable $X$. So we obtain 
\begin{align*}
    \;& \E\Big[\|\tildeF(\vu_{k}) - F(\vu_{k})\|^2 ~|~ \gF_{k} \Big] \\
    \leq \;& (1 - p_k)\|F(\vu_{k - 1}) - \tildeF(\vu_{k - 1})\|^2 + (1 - p_k)\frac{(n - b)}{b(n - 1)}\sum_{i=1}^n\frac{1}{n}\|F_i(\vu_k) - F_i(\vu_{k - 1})\|^2. 
\end{align*}
Taking expectation w.r.t.\ all the randomness on both sides and using the law of total expectation, we finally have  the result.
\end{proof}

\section{Omitted Proofs From Section~\ref{sec:expected-cocoercive} (Cocoercive Settings)}\label{appx:cocoercive}
We first prove the following lemma that is useful for our proof of Lemma~\ref{lem: cocoercive_decrease}.
\begin{lemma}\label{lem: coco_ineq}
For $k\geq 1$, let
\begin{align*}
\eta &= \frac{1}{4L},~~~\lambda_k = \frac{2}{k+4}, ~~~ c_k = \frac{(\sqrt{n}+2)(k+4)}{4L}, ~~~b= \lceil \sqrt{n}\rceil,~~~
p_{k + 1} = \begin{cases}
            \frac{4}{k + 5} \quad &  \forall k \leq \sqrt{n} \\
            \frac{4}{\sqrt{n} + 5} \quad & \forall k \geq \sqrt{n}
        \end{cases}
\end{align*}
as in Alg.~\ref{alg: cocoercive} and Lemma~\ref{lem: cocoercive_decrease}. Then, for $k\geq 1$ it holds that
    \begin{subequations}\label{eq:coco-ineq-01}
\begin{align}
    \frac{c_{k + 1}(1 - p_{k + 1})}{1 - \lambda_k} + \frac{2\eta}{\lambda_k} - c_k \leq \;& 0, \label{eq:coco-ineq-01a}\\
    \frac{2\eta}{\lambda_k} + \frac{c_{k + 1}(1 - p_{k + 1})(n - b)}{b(n - 1)(1 - \lambda_k)} - \frac{1}{\lambda_k L} \leq \;& 0.\label{eq:coco-ineq-01b}
\end{align}
\end{subequations}
\end{lemma}
\begin{proof}
By the definition of $b$, we have $b \geq \sqrt{n}$ which implies
\begin{align*}
    \frac{c_{k + 1}(1 - p_{k + 1})(n - b)}{b(n - 1)(1 - \lambda_k)} \leq \frac{c_{k + 1}(1 - p_{k + 1})(n - \sqrt{n})}{\sqrt{n}(n - 1)(1 - \lambda_k)} = \frac{c_{k + 1}(1 - p_{k + 1})}{(\sqrt{n} + 1)(1 - \lambda_k)}. 
\end{align*}
By this inequality and $\eta = \frac{1}{4L}$ for simplicity, the following inequality suffices to guarantee Eq.~\eqref{eq:coco-ineq-01b}
\begin{align*}
    \frac{c_{k + 1}(1 - p_{k + 1})}{1 - \lambda_k} \leq \frac{\sqrt{n} + 1}{2\lambda_k L}. 
\end{align*}
On the other hand, by Eq.~\eqref{eq:coco-ineq-01a}, we have 
\begin{align*}
    \frac{c_{k + 1}(1 - p_{k + 1})}{1 - \lambda_k} \leq c_k - \frac{1}{2\lambda_k L}.
\end{align*}
To let these two inequalities hold simultaneously, we take $c_k - \frac{1}{2\lambda_k L} = \frac{\sqrt{n} + 1}{2\lambda_k L}$ for simplicity and obtain $c_k = \frac{\sqrt{n} + 2}{2\lambda_k L}$, thus $c_{k + 1} = \frac{\sqrt{n} + 2}{2\lambda_{k + 1} L}$. Then the above two inequalities are equivalent to 
\begin{align}\label{eq:coco-ineq-1}
    1 - p_{k + 1} &\leq \frac{(\sqrt{n}+1)(1-\lambda_k)}{2\lambda_k L c_{k+1}} \notag \\
    &=\frac{\sqrt{n} + 1}{\sqrt{n} + 2}\frac{\lambda_{k + 1}(1 - \lambda_k)}{\lambda_k} \notag \\
    &= \frac{\sqrt{n} + 1}{\sqrt{n} + 2}\frac{k + 2}{k + 5}, 
\end{align}
where we plug in the definition of $c_{k+1}$ for the first equality and the definition of $\lambda_k$ for the last equality. When $k \leq \sqrt{n}$, we have $\frac{\sqrt{n} + 1}{\sqrt{n} + 2} \geq \frac{k + 1}{k + 2}$, then it suffices to choose $p_{k + 1} = \frac{4}{k + 5}$ to ensure Eq.~\eqref{eq:coco-ineq-1} holds. When $k \geq \sqrt{n}$, we have $\frac{k + 2}{k + 5} \geq \frac{\sqrt{n} + 2}{\sqrt{n} + 5}$, and it suffices to take $p_{k + 1} = \frac{4}{\sqrt{n} + 5}$. This is the definition of $p_{k+1}$ and the proof is completed.
\end{proof}

\ckCoco*
\begin{proof}
    By Assumption~\ref{asp: asp3} on $F$ and the monotonicity of $G$, we have  for $k\geq 1$
\begin{align*}
    \frac{1}{nL}\sum_{i=1}^n\|F_i(\vu_{k + 1}) - F_i(\vu_{k})\|^2 \leq \innp{F(\vu_{k + 1}) + \vg_{k + 1}, \vu_{k + 1} - \vu_{k}} - \innp{F(\vu_{k}) + \vg_k, \vu_{k + 1} - \vu_{k}},
\end{align*}
where $\vg_{k+1} \in G(\vu_{k+1})$ and $\vg_k \in G(\vu_k)$.
Dividing both sides by $\lambda_k$, we get
\begin{align}
    \;& \frac{1}{\lambda_k nL}\sum_{i=1}^n\|F_i(\vu_{k + 1}) - F_i(\vu_{k})\|^2 \notag \\
    \leq \;& \frac{1}{\lambda_k}\innp{F(\vu_{k + 1}) + \vg_{k + 1}, \vu_{k + 1} - \vu_{k}} - \frac{1}{\lambda_k}\innp{F(\vu_{k}) + \vg_k, \vu_{k + 1} - \vu_{k}}.\label{eq: sfr4}
\end{align}
Recall that by the definition of the resolvent operator and the definition of $\vu_{k+1}$, we have for $k\geq 1$
\begin{align*}
    \vg_{k + 1} = \frac{1}{\eta}\big(\lambda_k\vu_0 + (1 - \lambda_k)\vu_k - \eta\tildeF(\vu_k) - \vu_{k + 1}\big) \in G(\vu_{k + 1}), 
\end{align*}
which lets us rewrite the algorithm updates as
\begin{equation*}
    \vu_{k + 1} = \lambda_{k}\vu_0 + (1 - \lambda_{k})\vu_k - \eta(\tildeF(\vu_k) + \vg_{k + 1}).
\end{equation*}
By simple rearrangements on this representation of $\vu_{k+1}$, we have for $k\geq 1$
\begin{subequations}
\begin{align}
    \vu_{k + 1} - \vu_k = \;& \lambda_{k}(\vu_0 - \vu_k) - \eta(\tildeF(\vu_{k}) + \vg_{k + 1}) \label{eq: sgr3}\\
    \vu_{k + 1} - \vu_k = \;& \frac{\lambda_{k}}{1 - \lambda_{k}}(\vu_0 - \vu_{k + 1}) - \frac{\eta}{1 - \lambda_{k}}(\tildeF(\vu_{k}) + \vg_{k + 1}).\label{eq: sgr4}
\end{align}
\end{subequations}
Plugging Eq.~\eqref{eq: sgr4} in the first term in the right-hand side of Eq.~\eqref{eq: sfr4} and Eq.~\eqref{eq: sgr3} in the second term in the right-hand side of Eq.~\eqref{eq: sfr4}, we obtain
\begin{align}
    \;& \frac{1}{\lambda_k nL}\sum_{i=1}^n\|F_i(\vu_{k + 1}) - F_i(\vu_{k})\|^2 \notag \\
    \leq \;& \frac{1}{1 - \lambda_{k}}\innp{F(\vu_{k + 1}) + \vg_{k + 1}, \vu_0 - \vu_{k + 1}}  - \frac{\eta}{\lambda_k(1 - \lambda_{k})}\innp{F(\vu_{k + 1}) + \vg_{k + 1}, \tildeF(\vu_k) + \vg_{k + 1}}  \notag \\
    & - \innp{F(\vu_{k}) + \vg_k, \vu_0 - \vu_{k}} + \frac{\eta}{\lambda_k}\innp{F(\vu_{k}) + \vg_k, \tildeF(\vu_k) + \vg_{k + 1}}.\label{eq: nbh2}
\end{align}
We next represent the second and fourth inner products in the right-hand side of Eq.~\eqref{eq: nbh2} with squared norms as 
\begin{align}
    \;& \frac{\eta}{\lambda_k}\innp{F(\vu_{k}) + \vg_k, \tildeF(\vu_k) + \vg_{k + 1}} \notag \\
    = \;& \frac{\eta}{2\lambda_k}\Big(\|F(\vu_{k}) + \vg_k\|^2 + \|\tildeF(\vu_k) + \vg_{k + 1}\|^2 - \|F(\vu_{k}) + \vg_k - \tildeF(\vu_k) - \vg_{k + 1}\|^2\Big) \notag \\
    \leq \;& \frac{\eta}{2\lambda_k}\Big(\|F(\vu_{k}) + \vg_k\|^2 + \|\tildeF(\vu_k) + \vg_{k + 1}\|^2 \Big)\label{eq: nbh3}
\end{align}
and 
\begin{align}\label{eq: swr4}
    \;& -\frac{\eta}{\lambda_k(1-\lambda_k)}\innp{F(\vu_{k + 1}) + \vg_{k + 1}, \tildeF(\vu_k) + \vg_{k + 1}} \notag \\
    = \;& -\frac{\eta}{2\lambda_k(1-\lambda_k)}\Big(\|F(\vu_{k + 1}) + \vg_{k + 1}\|^2 + \|\tildeF(\vu_k) + \vg_{k + 1}\|^2- \|F(\vu_{k + 1}) - \tildeF(\vu_k)\|^2\Big).
\end{align}
We now estimate the second term on the right-hand side of Eq.~\eqref{eq: swr4} by Young's inequality as
\begin{align*}
     -\frac{\eta}{2\lambda_k(1-\lambda_k)} \| \tildeF(\vu_k) + \vg_{k+1}\|^2
    = \;& -\left(\frac{\eta}{2\lambda_k} + \frac{\eta}{2(1-\lambda_k)} \right)\| \tildeF(\vu_k) + \vg_{k+1}\|^2 \\ 
    \leq \;& -\frac{\eta}{2\lambda_k} \| \tildeF(\vu_k) + \vg_{k+1} \|^2 - \frac{\eta}{4(1-\lambda_k)} \| F(\vu_{k+1}) + \vg_{k+1} \|^2 \\
    &\quad+ \frac{\eta}{2(1-\lambda_k)} \| F(\vu_{k+1})-\tildeF(\vu_k) \|^2.
\end{align*}
We use this estimation in Eq.~\eqref{eq: swr4} and combine like terms by also using the definition of $\lambda_k$ to get
\begin{align}
    \;& -\frac{\eta}{\lambda_k(1-\lambda_k)}\innp{F(\vu_{k + 1}) + \vg_{k + 1}, \tildeF(\vu_k) + \vg_{k + 1}} \notag \\
    \leq \;& -\frac{\eta}{2\lambda_{k+1}(1-\lambda_k)} \| F(\vu_{k+1}) + \vg_{k+1}\|^2 - \frac{\eta}{2\lambda_k} \| \tildeF(\vu_k) + \vg_{k+1}\|^2 \notag \\
    & + \frac{\eta(1+\lambda_k)}{2\lambda_{k}(1-\lambda_k)} \| F(\vu_{k+1}) - \tildeF(\vu_k)\|^2 \notag \\
    \leq \;&  -\frac{\eta}{2\lambda_{k+1}(1-\lambda_k)} \| F(\vu_{k+1}) + \vg_{k+1}\|^2 - \frac{\eta}{2\lambda_k} \| \tildeF(\vu_k) + \vg_{k+1}\|^2   \notag \\
    & + \frac{2\eta}{n\lambda_{k}} \sum_{i=1}^n\| F_i(\vu_{k+1}) - F_i(\vu_k) \|^2+ \frac{2\eta}{\lambda_{k}} \| F(\vu_k) - \tildeF(\vu_k)\|^2,\label{eq: nbh4}
\end{align}
where the last step is by Young's inequality, Jensen's inequality and $\frac{1+\lambda_k}{1-\lambda_k} \leq 2$.

Combining Eq.~\eqref{eq: nbh3} and Eq.~\eqref{eq: nbh4}, plugging into Eq.~\eqref{eq: nbh2}, and rearranging the terms gives
\begin{align*}
    \;& \frac{1}{1-\lambda_k}\left( \frac{\eta}{2\lambda_{k+1}} \| F(\vu_{k+1}) + \vg_{k+1} \|^2 + \innp{F(\vu_{k+1}) + \vg_{k+1}, \vu_{k+1} - \vu_0} \right) \notag \\
    \leq \;& \frac{\eta}{2\lambda_{k}} \| F(\vu_{k}) + \vg_{k} \|^2 + \innp{F(\vu_{k}) + \vg_{k}, \vu_{k} - \vu_0} \\
    & + \frac{1}{n\lambda_k}\left(2\eta - \frac{1}{L}\right) \sum_{i=1}^n\| F_i(\vu_{k+1}) - F_i(\vu_k) \|^2 + \frac{2\eta}{\lambda_{k}} \| F(\vu_k) - \tildeF(\vu_k)\|^2.
\end{align*}
Adding $\frac{c_{k + 1}}{1 - \lambda_k}\|F(\vu_{k + 1}) - \tildeF(\vu_{k + 1})\|^2-c_k\|F(\vu_k) - \tildeF(\vu_{k})\|^2$ to both sides and using the definition of $\gC_k$, we obtain 
\begin{align}
    \frac{1}{1 - \lambda_k}\gC_{k + 1}
    \leq \;& \gC_k + \frac{c_{k + 1}}{1 - \lambda_k}\|F(\vu_{k + 1}) - \tildeF(\vu_{k + 1})\|^2 + \left(\frac{2\eta}{\lambda_k} - c_k\right)\|F(\vu_k) - \tildeF(\vu_k)\|^2 \notag \\
    & + \frac{1}{n\lambda_k}\left(2\eta- \frac{1}{L}\right)\sum_{i=1}^n\|F_i(\vu_{k + 1}) - F_i(\vu_{k})\|^2.\label{eq: byt4}
\end{align}
We recall the result of Lemma~\ref{lem:PAGE} which  states for $k\geq 1$
\begin{align}
     \E[\|F(\vu_{k + 1}) - \tildeF(\vu_{k + 1})\|^2] 
    &\leq (1 - p_{k + 1})\mathbb{E}[\|F(\vu_k) - \tildeF(\vu_k)\|^2] \notag\\&\quad+ \frac{(1 - p_{k + 1})(n - b)}{b(n - 1)}\mathbb{E}\left[\sum_{i=1}^n\frac{1}{n}\|F_i(\vu_{k + 1}) - F_i(\vu_k)\|^2\right].\label{eq: glo4}
\end{align}
We take the expectation of both sides of~\eqref{eq: byt4}, and then upper bound the resulting second term on the right-hand side of~\eqref{eq: byt4} by~\eqref{eq: glo4}. As a result, we  have for $k\geq 1$ that
\begin{align}
    \frac{1}{1 - \lambda_k}\;&\E[\gC_{k + 1}] 
    \leq \E[\gC_k] + \left( \frac{2\eta}{\lambda_k} - c_k + \frac{c_{k + 1}(1 - p_{k + 1})}{1 - \lambda_k}\right)\mathbb{E}\|F(\vu_k) - \tildeF(\vu_k)\|^2 \notag \\
    & + \left(\frac{2\eta}{\lambda_k} - \frac{1}{\lambda_k L}+ \frac{c_{k + 1}(1 - p_{k + 1})(n - b)}{b(n - 1)(1 - \lambda_k)}\right)\mathbb{E}\left[\sum_{i=1}^n\frac{1}{n}\|F_i(\vu_{k + 1}) - F_i(\vu_k)\|^2\right]. \label{eq: seh5}
\end{align}
By Lemma~\ref{lem: coco_ineq}, we have that the second and third terms on the right-hand side of~\eqref{eq: seh5} are nonpositive. Hence, we get the result after multiplying both sides by $1-\lambda_k$.
\end{proof}

\thmCoco*
\begin{proof}
After iterating the result of Lemma~\ref{lem: cocoercive_decrease}, we have
    \begin{align*}
    \E[\gC_k] \leq \Big(\prod_{i = 1}^{k - 1}(1 - \lambda_i)\Big)\E[\gC_1]. 
\end{align*}
Since $\lambda_i = \frac{2}{i + 4}$, we have 
\begin{align*}
    \prod_{i = 1}^{k - 1}(1 - \lambda_i) = \prod_{i = 1}^{k - 1}\frac{i+2}{i + 4} = \frac{(k + 1)!/2!}{(k + 3)! / 4!} = \frac{12}{(k+2)(k + 3)}, 
\end{align*}
which leads to 
\begin{align}
    \E[\gC_k] \leq \frac{12}{(k+2)(k + 3)}\E[\gC_1].\label{eq: shg5}
\end{align}
Next, we bound $\E[\gC_1]$, recalling the definition in~\eqref{eq: ck_def}. First note that $\widetilde F(\vu_1) = F(\vu_1)$, we know that $\gC_1$ does not involve any randomness, i.e., $\E[\gC_1] = \gC_1$, and have
\begin{align*}
    \gC_1 = \frac{\eta}{2\lambda_1}\|F(\vu_1) + \vg_1\|^2 + \innp{F(\vu_1) + \vg_1, \vu_1 - \vu_0}.
\end{align*}
Further, by  the definition of $\vu_1$, we have
\begin{align*}
    \vu_1 = J_{\frac{\eta}{2\lambda_1} G}\left(\vu_0 - \frac{\eta}{2\lambda_1} F(\vu_0)\right) = \vu_0 - \frac{\eta}{2\lambda_1} F(\vu_0) - \frac{\eta}{2\lambda_1} \vg_1.
\end{align*}
With this, we obtain 
\begin{align*}
    \gC_1 = \;& \frac{\eta}{2\lambda_1}\|F(\vu_1) + \vg_1\|^2 + \innp{F(\vu_1) + \vg_1, \vu_1 - \vu_0} \\
    = \;& \frac{\eta}{2\lambda_1}\|F(\vu_1) + \vg_1\|^2 - \frac{\eta}{2\lambda_1}\innp{F(\vu_1) + \vg_1, F(\vu_0) + \vg_1}.
\end{align*}
Decomposing the inner product term above, by adding and subtracting $F(\vu_1)$ in the second argument, we obtain 
\begin{align*}
    \gC_1 = \;& \frac{\eta}{2\lambda_1}\|F(\vu_1) + \vg_1\|^2 - \frac{\eta}{2\lambda_1}\|F(\vu_1) + \vg_1\|^2 - \frac{\eta}{2\lambda_1}\innp{F(\vu_1) + \vg_1, F(\vu_0) - F(\vu_1)} \\
    = \;& \frac{\eta}{2\lambda_1}\innp{F(\vu_1) + \vg_1, F(\vu_1) - F(\vu_0)}.
\end{align*}
Plugging in the definition $\vg_1 = \frac{2\lambda_1}{\eta}\left( \vu_0 - \vu_1\right) -F(\vu_0)$, we obtain 
\begin{align*}
    \gC_1 = \;& \frac{\eta}{2\lambda_1}\innp{F(\vu_1) - F(\vu_0) + \frac{2\lambda_1}{\eta}(\vu_0 - \vu_1), F(\vu_1) - F(\vu_0)} \\
    = \;& \frac{\eta}{2\lambda_1}\|F(\vu_1) - F(\vu_0)\|^2 - \innp{F(\vu_1) - F(\vu_0), \vu_1 - \vu_0} \\
    \overset{(\romannumeral1)}{\leq} \;& 0, 
\end{align*}
where $(\romannumeral1)$ is by $\frac{1}{L}$-cocoercivity of $F$ and $\frac{\eta}{2\lambda_1} = \frac{5}{16L}<\frac{1}{L}$. So we obtain $\E[\gC_k] \leq 0$ in view of~\eqref{eq: shg5}.
Recalling the definition of $\gC_k$ and noticing the term $c_k\|F(\vu_k) - \tildeF(\vu_k)\|^2$ is nonnegative, we have 
\begin{align*}
    \E\left[\frac{\eta(k + 4)}{4}\|F(\vu_k) + \vg_k\|^2 + \innp{F(\vu_k) + \vg_k, \vu_k - \vu_0}\right] \leq 0.
\end{align*}
Since $\vu_*$ is a solution to~\eqref{eq: prob1}, there exists $\vg_*$ such that $\vg_* \in G(\vu_*)$ and $F(\vu_*) + \vg_* = \mathbf{0}$, then we have 
\begin{align*}
    \innp{F(\vu_k) + \vg_k, \vu_k - \vu_0} = \;& \innp{F(\vu_k) + \vg_k, \vu_k - \vu_*} + \innp{F(\vu_k) + \vg_k, \vu_* - \vu_0} \\
    \overset{(\romannumeral1)}{\geq} \;& \innp{F(\vu_k) + \vg_k, \vu_* - \vu_0} \\
    \overset{(\romannumeral2)}{\geq} \;& -\|F(\vu_k) + \vg_k\|\|\vu_0 - \vu_*\|, 
\end{align*}
where we use the monotonicity of $F+G$ for $(\romannumeral1)$ and Cauchy-Schwarz inequality for $(\romannumeral2)$. Noticing that $\E[\|F(\vu_k) + \vg_k\|] \leq \big(\E[\|F(\vu_k) + \vg_k\|^2]\big)^{1/2}$ by Jensen's inequality, then we have 
\begin{align*}
    \frac{\eta(k + 4)}{4}\E[\|F(\vu_k) + \vg_k\|^2] \leq \|\vu_0 - \vu_*\|\big(\E[\|F(\vu_k) + \vg_k\|^2]\big)^{1/2}
\end{align*}
Completing the square and then solving for the quadratic gives
\begin{align*}
    \big(\E[\mathrm{Res}^2_{F + G}(\vu_k)]\big)^{1/2} = \big(\E[\|F(\vu_k) + \vg_k\|^2]\big)^{1/2} \leq \frac{16L\|\vu_0 - \vu_*\|}{k + 4}. 
\end{align*}
Given $\varepsilon > 0$, to return a point $\vu_K$ such that $\big(\E[\|F(\vu_K) + \vg_K\|^2]\big)^{1/2} \leq \varepsilon$, which also guarantees that $\E[\|F(\vu_K) + \vg_K\|] \leq \varepsilon$ by Jensen's inequality, the total number of iterations required is $K =\lceil \frac{16L\|\vu_0 - \vu_*\|}{\varepsilon}\rceil$. To obtain the total number of stochastic oracle queries, we let $m_{k}$ be the number of individual operator evaluations at iteration $k$ (in which we compute $\vu_{k + 1}$ and $\tildeF(\vu_{k + 1})$ as Alg.~\ref{alg: cocoercive}) for $k \geq 1$, and $M = 2n + \sum_{k = 1}^{K - 1} m_{k}$ be the total number of individual operator evaluations to return $\vu_K$. Conditional on the natural filtration $\gF_k$ that contains all algorithm randomness up to and not including iteration $k$, we have for $k \geq 1$
\begin{equation*}
    \E[m_{k}~|~\gF_{k}] = p_{k + 1} n + (1 - p_{k + 1})2b = \begin{cases}
        \frac{4}{k + 5}n + 2\frac{k + 1}{k + 5}\lceil\sqrt{n}\rceil \quad & \text{if } k \leq \sqrt{n}, \\
        \frac{4}{\sqrt{n} + 5}n + 2\frac{\sqrt{n} + 1}{\sqrt{n} + 5}\lceil\sqrt{ n}\rceil \quad & \text{if } k \geq \sqrt{n}.
    \end{cases}
\end{equation*}
Taking expectation w.r.t.\ all randomness and summing from $k = 1$ to $K$, we obtain 
\begin{equation*}
\begin{aligned}
    \E[M] = \;& 2n+ \E\Big[\sum_{k = 1}^{K - 1} m_{k}\Big] \\
    = \;& 2n + \sum_{k = 1}^{\lfloor \sqrt{n} \rfloor}\E[m_{k}] + \sum_{k = \lceil \sqrt{n} \rceil}^{K - 1}\E[m_{k}] \\
    \leq \;& 2n + 4n\sum_{k = 1}^{\lfloor \sqrt{n} \rfloor}\frac{1}{k + 5} + 2\lceil \sqrt{n} \rceil^2 + 6\sqrt{n}(K - \sqrt{n}) \\
    \leq \;& 4n + 4\sqrt{n} + 2 + 4n\log(\sqrt{n} + 5) + \frac{96\sqrt{n}L\|\vu_0 - \vu_*\|}{\varepsilon} = \widetilde{\gO}\Big(n + \frac{\sqrt{n}L\|\vu_0 - \vu_*\|}{\varepsilon}\Big), 
\end{aligned}
\end{equation*}
thus completing the proof.
\end{proof}

\section{Omitted Proofs From Section~\ref{sec:lipschitz} (Monotone Lipschitz Settings)}\label{appx:lipschitz}
\pcTores*
\begin{proof}
By the definition of resolvent operator, we have 
\begin{equation*}
    \vu - \Bar{\vu} \in \eta F(\bar \vu) + \eta G(\bar \vu) \iff \vu \in \vu - P^{\eta}(\vu) + \eta F(\Bar{\vu}) + \eta G(\Bar{\vu}) \iff \mathbf{0} \in \eta F(\Bar{\vu}) + \eta G(\Bar{\vu}) - P^{\eta}(\vu).
\end{equation*}
So we have $\|\eta F(\Bar{\vu}) + \eta \Bar{\vg}\| \leq \eta\varepsilon$, thus $\|F(\Bar{\vu}) + \Bar{\vg}\| \leq \varepsilon$, where $\bar\vg = \frac{1}{\eta}(\vu-\bar\vu)-F(\bar\vu)\in G(\bar\vu)$.
\end{proof}

In the rest of this section, for readability, we first provide, in Section~\ref{appx:resolvent}, the proofs for the convergence of Alg.~\ref{alg: forb_vr_sc} from Section~\ref{sec:resolvent}. Then, in Section~\ref{appx:inexact-halpern}, we give the proofs for the convergence  of inexact Halpern iteration from Section~\ref{sec:inexact-halpern}.
\subsection{Approximating the Resolvent}\label{appx:resolvent}
\begin{lemma}\label{lem: feta}
Let $F: \R^d \rightarrow \R^d$ be monotone and $L_Q$-Lipschitz in expectation as in Assumption~\ref{asp: asp2}. Then for $\vu^+ \in \R^d$, $\Bar{F}^\eta(\vu; \vu^+) = \eta F(\vu) + \vu - \vu^+$ defined in Eq.~\eqref{eq:barF} is $1$-strongly monotone and $(\eta L_Q + 1)$-average Lipschitz w.r.t.\ $\vu$.
\end{lemma}
\begin{proof}
For $\vu^+ \in \R^d$, strong monotonicity clearly follows since  for any $\vu, \vv \in \R^d$, we have
\begin{equation*}
    \innp{\Bar{F}^{\eta}(\vu; \vu^+) - \Bar{F}^{\eta}(\vv; \vu^+), \vu - \vv} = \eta \innp{F(\vu) - F(\vv), \vu - \vv} + \|\vu - \vv\|^2 \geq \|\vu - \vv\|^2, 
\end{equation*}
where we use monotonicity of $F$ for the last inequality. 
Further, since $F$ is $L$-average Lipschitz with oracle $F_\xi$ as Assumption~\ref{asp: asp2}, we have $\bar{F}_{\xi}^{\eta}(\vu; \vu^+) = \eta F_{\xi}(\vu) + \vu - \vu^+$ such that $\E_{\xi \sim Q}[\bar{F}_{\xi}^{\eta}(\vu; \vu^+)] = \eta F(\vu) + \vu - \vu^+ = \Bar{F}^{\eta}(\vu; \vu^+)$. Then for any $\vu, \vv \in \R^d$, 
\begin{equation*}
\begin{aligned}
    \;& \E_{\xi \sim Q}[\|\bar{F}_{\xi}^{\eta}(\vu; \vu^+) - \bar{F}_{\xi}^{\eta}(\vv; \vu^+)\|^2] \\
    = \;& \E_{\xi \sim Q}[\|\eta F_{\xi}(\vu) - \eta F_{\xi}(\vv) + \vu - \vv\|^2] \\
    = \;& \eta^2\E_{\xi \sim Q}[\|F_{\xi}(\vu) - F_{\xi}(\vv)\|^2] + \|\vu - \vv\|^2 + 2\eta\E_{\xi \sim Q}[\innp{F_{\xi}(\vu) - F_{\xi}(\vv), \vu - \vv}].
\end{aligned}
\end{equation*}
Using $L$-expected Lipschitzness of $F$ and Cauchy-Schwarz inequality for the quantity above, we obtain 
\begin{align*}
    \;& \E_{\xi \sim Q}[\|\bar{F}_{\xi}^{\eta}(\vu; \vu^+) - \bar{F}_{\xi}^{\eta}(\vv; \vu^+)\|^2] \\
    \leq \;& (\eta^2 L^2 + 1)\|\vu - \vv\|^2 + 2\eta\|\vu - \vv\|\E_{\xi \sim Q}[\|F_{\xi}(\vu) - F_{\xi}(\vv)\|] \\
    \overset{(\romannumeral1)}{\leq} \;& (\eta^2 L^2 + 2\eta L + 1)\|\vu - \vv\|^2 = (\eta L + 1)^2\|\vu - \vv\|^2.
\end{align*}
where $(\romannumeral1)$ is due to $\E_{\xi \sim Q}[\|F_{\xi}(\vu) - F_{\xi}(\vv)\|] \leq L\|\vu - \vv\|$ by Assumption~\ref{asp: asp2} and Jensen's inequality. Hence, $\bar{F}^{\eta}(\vu; \vu^+)$ is $(\eta L + 1)$-expected Lipschitz.
\end{proof}
We now move to the convergence proof for $\mathtt{VR{-}FoRB}$ which incorporates the loopless SVRG variance reduction technique~\citep{kovalev2020don,hofmann2015variance} into FoRB method~\citep{malitsky2020forward}.
\FoRB*
\begin{proof}
First, following the proof of~\cite[Theorem 22]{alacaoglu2021stochastic} with the addition of strong monotonicity of $A+B$ (\emph{cf.} the equation after Eq.~(56) therein using only monotonicity), and taking expectation w.r.t.\ all randomness on both sides, we plug in our parameter choices $\tau = \frac{\sqrt{p(1 - p)}}{2L_A}$ and $\alpha = 1 - p$, and then obtain
\begin{equation}\label{eq: grt3}
\begin{aligned}
    \;& (1-p+2\tau \mu)\mathbb{E}[\| \vv_{k+1} - \vv_*\|^2] + \mathbb{E}[\|\vw_{k+1} - \vv_*\|^2] + p \mathbb{E}[\| \vv_{k+1} - \vw_k\|^2] \\
    & + 2\tau \mathbb{E}[\langle A(\vv_{k+1}) - A(\vw_k), \vv_* - \vv_{k+1} \rangle] \\
    \leq \;& (1-p)\mathbb{E}[\| \vv_{k} - \vv_*\|^2] + \mathbb{E}[\|\vw_{k} - \vv_*\|^2] + p \mathbb{E}[\| \vv_{k} - \vw_{k-1}\|^2] \\
    & + 2\tau \mathbb{E}[\langle A(\vv_{k}) - A(\vw_{k-1}), \vv_* - \vv_{k} \rangle] - \frac{p}{2}\mathbb{E}[\| \vv_k - \vw_{k-1}\|^2] - \frac{1-p}{2} \E[\| \vv_{k+1} - \vv_k\|^2].
\end{aligned}
\end{equation}
To simplify the notation, we define
\begin{align*}
    a_k = \;& \frac{1}{2} \mathbb{E}[\| \vv_{k} - \vv_*\|^2], \\
    b_k = \;& \frac{1-p}{2} \mathbb{E}[\| \vv_k - \vv_*\|^2] + \mathbb{E}[\| \vw_k - \vv_*\|^2] + p\mathbb{E}[\| \vv_k - \vw_{k-1}\|^2] \\
    & + 2\tau \mathbb{E}[\langle A(\vv_k) - A(\vw_{k-1}), \vv_* - \vv_k \rangle].
\end{align*}
We first note that $b_k \geq 0$. Indeed, using Young's inequality, Lipschitzness of $A$ and the definition of $\tau$, we have 
\begin{align}
    |2\tau \langle A(\vv_k) - A(\vw_{k-1}), \vv_* - \vv_k \rangle| \leq \;& 2\tau L_A \|\vv_k - \vw_{k-1}\| \|\vv_k - \vv_*\| \notag  \\
    \leq \;& \frac{4\tau^2 L_A^2}{2(1-p)} \| \vv_{k} - \vw_{k-1}\|^2 + \frac{1-p}{2} \| \vv_k - \vv_*\|^2\notag \\
    = \;& \frac{p}{2} \| \vv_{k} - \vw_{k-1}\|^2 + \frac{1-p}{2} \| \vv_k - \vv_*\|^2.\label{eq: sfr3}
\end{align}
Then we have, by plugging in the definitions of $a_k, b_k$ in Eq.~\eqref{eq: grt3} and discarding the last term on the right-hand side, that 
\begin{align}
    (1 - p + 4\tau\mu) a_{k+1} + b_{k+1} \leq (1-p) a_k + b_k - \frac{p}{2} \E[\|\vv_k - \vw_{k-1}\|^2]. \label{eq: slo4}
\end{align}
For a constant $c > 0$ to be set later, let us write the right-hand side of Eq.~\eqref{eq: slo4} as 
\begin{align}
    (1-p)a_k + b_k = \;& (1-p)(1+c) a_k + (1-c) b_k + c\mathbb{E}[\|\vw_k - \vv_*\|^2] + pc\mathbb{E}[\|\vv_k - \vw_{k-1}\|^2] \notag \\
    & + 2\tau  c\mathbb{E}[\langle A(\vv_k) - A(\vw_{k-1}), \vv_* - \vv_k \rangle].\label{eq: vbn4}
\end{align}
By using the definition of $\vw_k$ and the tower rule, we have
\begin{align*}
    c\mathbb{E}[\|\vw_k - \vv_*\|^2] \leq \;& 2c\mathbb{E}[\|\vv_k - \vv_*\|^2] + 2c\mathbb{E}[\| \vv_k - \vw_{k}\|^2] \\
    = \;& 2c\mathbb{E}[\|\vv_k - \vv_*\|^2] + 2c(1-p)\mathbb{E}[\| \vv_k - \vw_{k-1}\|^2].
\end{align*}
Combining this estimate with Eq.~\eqref{eq: sfr3}, and using trivial facts that $cp \leq c$ and $1 - p \leq 1$ to further manipulate Eq.~\eqref{eq: vbn4} gives
\begin{align}
    (1-p)a_k + b_k \leq \;& (1-p)(1+c) a_k + (1-c) b_k + \frac{5c}{2} \mathbb{E}[\| \vv_k - \vv_*\|^2] + \frac{7c}{2} \mathbb{E}[\| \vv_{k} - \vw_{k-1}\|^2] \notag \\
    \leq \;& (1-p+6c)a_k + (1-c)b_k + \frac{7c}{2} \mathbb{E}[\| \vv_{k} - \vw_{k-1}\|^2].\label{eq: sfr5}
\end{align}
We now let $\frac{7c}{2} \leq \frac{p}{2}$ and use the estimation Eq.~\eqref{eq: sfr5} in Eq.~\eqref{eq: slo4} to get
\begin{equation}\label{eq: nhf4}
    (1 - p + 4\tau\mu)a_{k+1} + b_{k+1} \leq (1-p+6c) a_k + (1-c)b_k.
\end{equation}
Let us also set $6c \leq 3\tau\mu$ and then
\begin{align*}
    (1 - p + 4\tau\mu)a_{k+1} + b_{k+1} \leq \;& (1 - p + 3\tau\mu) a_k + (1-c)b_k \\
    = \;& \Big( 1-\frac{\tau\mu}{1 - p + 4\tau\mu} \Big)(1 - p + 4\tau\mu) a_k + (1-c)b_k \\
    \leq \;& \Big( 1-\min\big\{ \frac{\tau\mu}{1 + 4\tau\mu}, c\big\} \Big)  \left((1 - p + 4\tau\mu) a_k + b_k \right),
\end{align*}
where the last step used nonnegativity of $b_k$. Let us also define $c=\min\left\{ \frac{p}{7}, \frac{\tau\mu}{2} \right\}$.
By iterating the inequality and using $b_{k} \geq 0$, we have 
\begin{align}\label{eq: dwe3}
    (1 - p + 4\tau\mu)a_{k} \leq \Big( 1-\min\big\{ \frac{\tau\mu}{1 + 4\tau\mu}, c\big\} \Big)^{k}\big((1 - p + 4\tau\mu)a_0 + b_0\big)
\end{align}
thus 
\begin{align*}
    \E[\|\vv_k - \vv_*\|^2] \leq 2\Big( 1-\min\big\{ \frac{\tau\mu}{1 + 4\tau\mu}, c\big\} \Big)^{k}\Big(a_0 + \frac{1}{1 - p + 4\tau\mu}b_0\Big).
\end{align*}
Further noticing $\vv_0 = \vw_0 = \vw_{-1}$, then we have 
\begin{align*}
    a_0 + \frac{1}{1 - p + 4\tau\mu}b_0 = \;& \frac{1}{2}\|\vv_0 - \vv_*\|^2 + \frac{1}{1 - p + 4\tau\mu}\Big(\frac{1 - p}{2}\|\vv_0 - \vv_*\|^2 + \|\vv_0 - \vv_*\|^2\Big) \\
    \overset{(\romannumeral1)}{\leq} \;& \Big(\frac{1}{2} + \frac{1}{2} + \frac{1}{1 - p}\Big)\|\vv_0 - \vv_*\|^2 = \frac{2 - p}{1 - p}\|\vv_0 - \vv_*\|^2, 
\end{align*}
where we use $1 - p + 4\tau\mu \geq 1 - p$ for $(\romannumeral1)$. So we obtain 
\begin{equation*}
    \mathbb{E}[\| \vv_{k} - \vv_*\|^2] \leq \Big( 1-\min\big\{ \frac{\tau\mu}{1 + 4\tau\mu}, c\big\} \Big)^{k} \frac{4-2p}{1-p}\| \vv_0 - \vv_*\|^2.
\end{equation*}
Given $\bar \varepsilon > 0$, we now see the exact number of iterations to guarantee $\E[\|\vv_k - \vu_*\|^2] \leq \bar \varepsilon^2$. We have 
\begin{align}
    \mathbb{E}[\| \vv_{k} - \vv_*\|^2] \leq \;& \Big( 1-\min\big\{ \frac{\tau\mu}{1 + 4\tau\mu}, c\big\} \Big)^{k} \frac{4-2p}{1-p}\| \vv_0 - \vv_*\|^2 \notag \\
    \leq \;& 6\exp\Big(-\min\big\{ \frac{\tau\mu}{1 + 4\tau\mu}, c\big\}k\Big)\| \vv_0 - \vv_*\|^2, \label{eq: vvf4}
\end{align}
where we use $p = \frac{1}{n}$ and assume $n \geq 2$ without loss of generality, thus 
\begin{align*}
    k \geq 2\max\left\{ \frac{1 + 4\tau\mu}{\tau\mu}, \frac{1}{c}\right\}\log\big(\frac{\sqrt{6}\|\vv_0 - \vv_*\|}{\bar \varepsilon}\big) = 2\max\left\{ \frac{1 + 4\tau\mu}{\tau\mu}, \frac{7}{p}, \frac{2}{\tau\mu}\right\}\log\big(\frac{\sqrt{6}\|\vv_0 - \vv_*\|}{\bar\varepsilon}\big).
\end{align*}
using $c=\min\left\{ \frac{p}{7}, \frac{\tau\mu}{2} \right\}$. Plugging in the choices that $p = \frac{1}{n}$ and $\tau = \frac{\sqrt{p(1 - p)}}{2L_A} = \frac{\sqrt{n - 1}}{2nL_A}$, we have 
\begin{align}\label{eq:forb-k}
    k \geq 2\max\left\{4 + \frac{2nL_A}{\mu\sqrt{n - 1}}, 7n, \frac{4nL_A}{\mu\sqrt{n - 1}}\right\}\log\big(\frac{\sqrt{6}\|\vv_0 - \vv_*\|}{\bar \varepsilon}\big), 
\end{align}
For simplicity, we assume without loss of generality that $n \geq 2$ and $L_A \geq 1$, then we have $\frac{nL_A}{\sqrt{n - 1}} \geq 2$ and $\frac{4nL_A}{\sqrt{n - 1}} \leq 7\sqrt{n}L_A$.
So it suffices to choose $k = \big\lceil 14\max\big\{n, \frac{\sqrt{n}L_A}{\mu}\big\}\log\big(\frac{\sqrt{6}\|\vv_0 - \vv_*\|}{\bar \varepsilon}\big)\big\rceil$. 
Further, we notice that Alg.~\ref{alg: forb_vr_sc} has constant per-iteration average cost $pn + 2 = 3$ by $p = \frac{1}{n}$, so we have the oracle complexity to be $\gO\Big(\big(n + \frac{\sqrt{n}L_A}{\mu}\big)\log\big(\frac{\|\vv_0 - \vv_*\|}{\bar \varepsilon}\big)\Big)$.
\end{proof}

\subsection{Inexact Halpern Iteration}\label{appx:inexact-halpern}
\begin{restatable}{lemma}{uk}\label{lemma:uk-rec}
Let $F$ be monotone and Lipschitz, and $G$ be maximally monotone. 
Assume that conditional on the algorithm randomness up to iteration k, we have $\E_k[\|\ve_k\|^2] \leq \frac{\|P^\eta(\vu_k)\|^2}{(k + 2)^8} = \frac{\|\vu_k - J_{\eta(F + G)}(\vu_k)\|^2}{(k + 2)^8}$ for all $k \geq 1$ and $\E[\|\ve_0\|^2] \leq \frac{\|P^\eta(\vu_0)\|^2}{27}$. Then the iterates $\vu_k$ generated by Alg.~\ref{alg:monotone} satisfy 
\begin{align}
    \E[\|\vu_k - \vu_*\|^2] \leq 2\|\vu_0 - \vu_*\|^2
\end{align}
where $\vu_*$ is the solution point such that $J_{\eta(F + G)}(\vu_*) = \vu_*$.
\end{restatable}
\begin{proof}
By Eq.~\eqref{eq:inexact-halpern}, definition of $P^\eta$ and noticing $J_{\eta(F + G)}(\vu_*) = \vu_*$, we have 
\begin{align*}
    \;& \|\vu_k - \vu_*\| \\
    = \;& \|\lambda_{k - 1}(\vu_0 - \vu_*) + (1 - \lambda_{k - 1})(J_{\eta(F + G)}(\vu_{k - 1}) - J_{\eta(F + G)}(\vu_*)) - (1 - \lambda_{k - 1})\ve_{k - 1}\| \\
    \overset{(\romannumeral1)}{\leq} \;& \lambda_{k - 1}\|\vu_0 - \vu_*\| + (1 - \lambda_{k - 1})\|J_{\eta(F + G)}(\vu_{k - 1}) - J_{\eta(F + G)}(\vu_*)\| - (1 - \lambda_{k - 1})\|\ve_{k - 1}\| \\
    \overset{(\romannumeral2)}{\leq} \;& \lambda_{k - 1}\|\vu_0 - \vu_*\| + (1 - \lambda_{k - 1})\|\vu_{k - 1} - \vu_*\| + (1 - \lambda_{k - 1})\|\ve_{k - 1}\|, 
\end{align*}
where we use the triangle inequality for $(\romannumeral1)$ and $(\romannumeral2)$ is due to nonexpansivity of resolvent operators. Iterating this inequality till $k = 1$, we obtain 
\begin{align*}
    \|\vu_k - \vu_*\| \leq \;& \Big(\underbrace{\lambda_{k - 1} + \prod_{i = 1}^k(1 - \lambda_{i - 1}) + \sum_{i = 1}^{k - 1} \big(\lambda_{i - 1}\prod_{j = i + 1}^k(1 - \lambda_{j - 1})\big)}_{\gT_{[1]}}\Big)\|\vu_0 - \vu_*\| \\
    & + \sum_{i = 1}^k \Big(\|\ve_{i - 1}\|\underbrace{\prod_{j = i}^k (1 - \lambda_{j - 1}))}_{\gT^{(i)}_{[2]}}\Big).
\end{align*}
Plugging in our choice that $\lambda_{i} = \frac{1}{i + 2}$ such that 
\begin{align*}
    \gT_{[1]} = \;& \frac{1}{k + 1} + \prod_{i = 1}^k\frac{i}{i + 1} + \sum_{i = 1}^{k - 1}\frac{1}{i + 1}\frac{i + 1}{k + 1} = 1, \quad \gT_{[2]}^{(i)} = \prod_{j = i}^k\frac{j}{j + 1} = \frac{i}{k + 1}, 
\end{align*}
 we then have 
\begin{equation*}
    \|\vu_k - \vu_*\| \leq \|\vu_0 - \vu_*\| + \frac{1}{k + 1}\sum_{i = 1}^k i\|\ve_{i - 1}\|.
\end{equation*}
Squaring the terms on both sides and taking expectation w.r.t.\ all randomness on both sides, we obtain 
\begin{align*}
    \E[\|\vu_k - \vu_*\|^2] \leq \;& \E\Big[\|\vu_0 - \vu_*\| + \frac{1}{k + 1}\sum_{i = 1}^k i\|\ve_{i - 1}\|\Big]^2 \\
    \overset{(\romannumeral1)}{\leq} \;& \frac{3}{2}\|\vu_0 - \vu_*\|^2 + \frac{3}{(k + 1)^2}\E\Big[\sum_{i = 1}^k i\|\ve_{i - 1}\|\Big]^2 \\
    \overset{(\romannumeral2)}{\leq} \;& \frac{3}{2}\|\vu_0 - \vu_*\|^2 + \frac{3}{k + 1}\E\Big[\sum_{i = 1}^k i^2\|\ve_{i - 1}\|^2\Big], 
\end{align*}
where we use Young's inequality for $(\romannumeral1)$ and the fact that $\big(\sum_{i = 1}^k x_i\big)^2 \leq k\sum_{i = 1}^k x_i^2$ for any $x_i \in \sR$ for $(\romannumeral2)$. 
Since $\E_k[\|\ve_k\|^2] \leq \frac{\|P^\eta(\vu_k)\|^2}{(k + 2)^8} = \frac{\|\vu_k - J_{\eta(F + G)}(\vu_k)\|^2}{(k + 2)^8}$ for all $k \geq 1$, thus $\E[\|\ve_k\|^2] \leq \frac{\E[\|P^\eta(\vu_k)\|^2]}{(k + 2)^8}$ by the law of total expectation, and $\E[\|\ve_0\|^2] \leq \frac{\|P^\eta(\vu_0)\|^2}{27}$, then we have  
\begin{align*}
    \E[\|\vu_k - \vu_*\|^2] \leq \;& \frac{3}{2}\|\vu_0 - \vu_*\|^2 + \frac{3}{k + 1}\sum_{i = 2}^k i^2\frac{\E[\|P^\eta(\vu_{i - 1})\|^2]}{(i + 2)^8} + \frac{\|P^\eta(\vu_0)\|^2}{9(k + 1)} \\
    \leq \;& \frac{3}{2}\|\vu_0 - \vu_*\|^2 + \frac{3}{k + 1}\sum_{i = 1}^{k - 1} \frac{\E[\|P^\eta(\vu_{i})\|^2]}{(i + 1)^6} + \frac{\|P^\eta(\vu_0)\|^2}{9(k + 1)}.
\end{align*}
Noticing that $P^\eta(\vu_*) = \mathbf{0}$ and $P^\eta$ is $2$-Lipschitz by nonexpansivity of the resolvent operator, we have for all $i \geq 0$
\begin{align*}
    \|P^\eta(\vu_i)\|^2 = \|P^\eta(\vu_i) - P^\eta(\vu_*)\|^2 \leq 4\|\vu_i - \vu_*\|^2, 
\end{align*}
which leads to 
\begin{align}
    \E[\|\vu_k - \vu_*\|^2] \leq \;& \frac{3}{2}\|\vu_0 - \vu_*\|^2 + \frac{12}{k + 1}\sum_{i = 1}^{k - 1} \frac{\E[\|\vu_i - \vu_*\|^2]}{(i + 1)^6} + \frac{4\|\vu_0 - \vu_*\|^2}{9(k + 1)}.\label{eq:uk-recursion}
\end{align}
We claim that for all $k \geq 1$: 
\begin{align*}
    \E[\|\vu_k - \vu_*\|^2] \leq 2\|\vu_0 - \vu_*\|^2. 
\end{align*}
We prove this claim by induction on $k$. First for the base case $k = 1$, we have by Eq.~\eqref{eq:uk-recursion} 
\begin{align*}
    \E[\|\vu_1 - \vu_*\|^2] \leq \frac{3}{2}\|\vu_0 - \vu_*\|^2 + \frac{2}{9}\|\vu_0 - \vu_*\|^2 \leq 2\|\vu_0 - \vu_*\|^2.
\end{align*}
Suppose that $\E[\|\vu_i - \vu_*\|^2] \leq 2\|\vu_0 - \vu_*\|^2$ for all $i \leq k - 1$, then for the case $k \geq 2$, we have by Eq.~\eqref{eq:uk-recursion} 
\begin{align*}
    \E[\|\vu_k - \vu_*\|^2] \leq \;& \frac{3}{2}\|\vu_0 - \vu_*\|^2 + \frac{12}{k + 1}\sum_{i = 1}^{k - 1} \frac{\E[\|\vu_i - \vu_*\|^2]}{(i + 1)^6} + \frac{4\|\vu_0 - \vu_*\|^2}{9(k + 1)} \\
    \leq \;& \frac{3}{2}\|\vu_0 - \vu_*\|^2 + \frac{24\|\vu_0 - \vu_*\|^2}{k + 1}\big(\sum_{i = 1}^{\infty} \frac{1}{i^6} - 1\big) +  \frac{4\|\vu_0 - \vu_*\|^2}{9(k + 1)} \\
    \overset{(\romannumeral2)}{\leq} \;& 2\|\vu_0 - \vu_*\|^2, 
\end{align*}
where we use $k \geq 2$ and $\sum_{i = 1}^{\infty} \frac{1}{i^6} = \frac{\pi^6}{945}$ for $(\romannumeral1)$. So our claim holds.
\end{proof}

\begin{restatable}{lemma}{ckMono}
\label{lemma:eck-change}
Let $F$ be monotone and Lipschitz and let $G$ be maximally monotone.
Assume that conditional on the algorithm randomness up to
iteration k, we have $\E_k[\|\ve_k\|^2] \leq \frac{\|P^\eta(\vu_k)\|^2}{(k + 2)^8}$ for all $k \geq 1$ and $\E[\|\ve_0\|^2] \leq \frac{\|P^\eta(\vu_0)\|^2}{27}$. Then for the iterates $\vu_k$ generated by Alg.~\ref{alg:monotone}, we have  
\begin{align}\label{eq:eck-change}
    \E[\gC_{k + 1}] - \E[\gC_k]  
    \leq \frac{16\|\vu_0 - \vu_*\|^2}{(k + 2)^2}, 
\end{align}
where 
\begin{align}
    \gC_k = \frac{k(k + 1)}{2}\|P^\eta(\vu_k)\|^2 + (k + 1)\innp{P^\eta(\vu_k), \vu_k - \vu_0}.
\end{align}
\end{restatable}
\begin{proof}
By Eq.~\eqref{eq:inexact-halpern} and the definition of $\ve_k$, we have 
\begin{align}
    \vu_{k + 1} - \vu_k = \;& \lambda_k(\vu_0 - \vu_k) - (1 - \lambda_k)P^\eta(\vu_k) - (1 - \lambda_k)\ve_k, \label{eq:uk-diff-1}\\
    \vu_{k + 1} - \vu_k = \;& \frac{\lambda_k}{1 - \lambda_k}(\vu_0 - \vu_{k + 1}) - P^\eta(\vu_k) - \ve_k. \label{eq:uk-diff-2}
\end{align}
Noticing that $P^\eta$ is $\frac{1}{2}$-cocoercive by the nonexpansivity of the resolvent operator and using the equations above, we have 
\begin{align*}
    \frac{1}{2}\|P^\eta(\vu_{k + 1}) - P^\eta(\vu_k)\|^2 \leq \;& \innp{P^\eta(\vu_{k + 1}) - P^\eta(\vu_k), \vu_{k + 1} - \vu_k} \\
    = \;& \frac{\lambda_k}{1 - \lambda_k}\innp{P^\eta(\vu_{k + 1}), \vu_0 - \vu_{k + 1}} - \innp{P^\eta(\vu_{k + 1}), P^\eta(\vu_k) + \ve_k} \\
    & - \lambda_k\innp{P^\eta(\vu_k), \vu_0 - \vu_k} + (1 - \lambda_k)\big(\|P^\eta(\vu_k)\|^2 + \innp{P^\eta(\vu_k), \ve_k}\big)
\end{align*}
Expanding the quadratic term $\|P^\eta(\vu_{k + 1}) - P^\eta(\vu_k)\|^2$ and rearranging the terms, we obtain 
\begin{align*}
    \;& \frac{1}{2}\|P^\eta(\vu_{k + 1})\|^2  + \frac{\lambda_k}{1 - \lambda_k}\innp{P^\eta(\vu_{k + 1}), \vu_{k + 1} - \vu_0} \\
    \leq \;& \big(\frac{1}{2} - \lambda_k\big)\|P^\eta(\vu_k)\|^2 + \lambda_k\innp{P^\eta(\vu_k), \vu_k - \vu_0} + \innp{\ve_k, (1 - \lambda_k)P^\eta(\vu_k) - P^\eta(\vu_{k + 1})}. 
\end{align*}
Plugging in the choice $\lambda_k = \frac{1}{k + 2}$ and multiplying $(k + 1)(k + 2)$ on both sides, then we can bound the consecutive change of $\gC_k$ as below 
\begin{align}
    \gC_{k + 1} - \gC_k = \;& \frac{(k + 1)(k + 2)}{2}\|P^\eta(\vu_{k + 1})\|^2 - \frac{k(k + 1)}{2}\|P^\eta(\vu_k)\|^2 \notag \\
    & + (k + 2)\innp{P^\eta(\vu_{k + 1}), \vu_{k + 1} - \vu_0} - (k + 1)\innp{P^\eta(\vu_{k}), \vu_{k} - \vu_0} \notag \\
    \leq \;& (k + 1)(k + 2)\innp{\ve_k, \frac{k + 1}{k + 2}P^\eta(\vu_k) - P^\eta(\vu_{k + 1})}. \label{eq:ck-change}
\end{align}
Noticing that $P^\eta(\vu_*) = \mathbf{0}$ and $P^\eta$ is $2$-Lipschitz by the nonexpansivity of the resolvent operator, we have 
\begin{align*}
    \innp{\ve_k, \frac{k + 1}{k + 2}P^\eta(\vu_k) - P^\eta(\vu_{k + 1})}
    = \;& \innp{\ve_k, \frac{k + 1}{k + 2}(P^\eta(\vu_k) - P^\eta(\vu_*)) - (P^\eta(\vu_{k + 1}) - P^\eta(\vu_*))} \\
    \overset{(\romannumeral1)}{\leq} \;& 2\|\ve_k\|\Big(\frac{k + 1}{k + 2}\|\vu_k - \vu_*\| + \|\vu_{k + 1} - \vu_*\|\Big)\\
    \overset{(\romannumeral2)}{\leq} \;& (k+2)^4\|\ve_k\|^2 + \frac{2}{(k+2)^4}\big( \|\vu_k - \vu_*\|^2 + \| \vu_{k+1}-\vu_*\|^2\big),\label{eq: vbn3}
\end{align*}
where we use Cauchy-Schwarz inequality and Lipschitzness of $P^\eta$ for $(\romannumeral1)$, and $(\romannumeral2)$ is due to $\frac{k + 1}{k + 2} \leq 1$ and Young's inequality. Taking expectation w.r.t.\ all randomness on both sides for Eq.~\eqref{eq:ck-change} and using the results from Lemma~\ref{lemma:uk-rec}, we have for $k \geq 1$
\begin{align*}
    \E[\gC_{k + 1}] - \E[\gC_k] \leq \;& (k + 1)(k + 2)\E\Big[\big\langle\ve_k, \frac{k + 1}{k + 2}P^\eta(\vu_k) - P^\eta(\vu_{k + 1})\big\rangle\Big] \\
    \leq \;& (k + 1)(k + 2)\E\Big[(k+2)^4\|\ve_k\|^2 + \frac{2}{(k+2)^4}\big(\|\vu_k - \vu_*\|^2 + \| \vu_{k+1}-\vu_*\|^2\big)\Big] \\
    \overset{(\romannumeral1)}{\leq} \;& \frac{4\E[\|\vu_k - \vu_*\|^2]}{(k + 2)^2} + \frac{2}{(k + 2)^2}\big(\E[\|\vu_k - \vu_*\|^2] + \E[\|\vu_{k + 1} - \vu_*\|^2]\big) \\
    \overset{(\romannumeral2)}{\leq} \;& \frac{16\|\vu_0 - \vu_*\|^2}{(k + 2)^2}, 
\end{align*}
where $(\romannumeral1)$ is due to $\E[\|\ve_k\|^2] \leq \frac{\|P^\eta(\vu_k)\|^2}{(k + 2)^8}$ and $\|P^\eta(\vu_k)\|^2 \leq 4\|\vu_k - \vu_*\|^2$ and we use $\E[\|\vu_k - \vu_*\|^2] \leq 2\|\vu_0 - \vu_*\|^2$ for $(\romannumeral2)$.
\end{proof}

\thmMono*
\begin{proof}
We first prove $\E_k[\|\ve_k\|^2] \leq \frac{\|P^{\eta}(\vu_k)\|^2}{(k + 2)^8}$ by our number of inner iterations of Alg.~\ref{alg:monotone}. Given $\varepsilon_k > 0$, noticing that $\|P^\eta(\vu_k)\|^2 = \|\widetilde{J}_{\eta(F + G)}(\vu_k) - J_{\eta(F + G)}(\vu_k)\|^2$ and using the convergence results in Theorem~\ref{thm:strongly-monotone}, we have our subsolver Alg.~\ref{alg: forb_vr_sc} with initial point $\vu_k$ return $\widetilde{J}_{\eta(F + G)}(\vu_k)$ such that $\E_k[\|\widetilde{J}_{\eta(F + G)}(\vu_k) - J_{\eta(F + G)}(\vu_k)\|^2] \leq \varepsilon_k^2$ in the number of iterations 
\begin{align*}
    \Big\lceil 14\max\big\{n, \sqrt{n}(\eta L + 1)\big\}\log\big(\frac{\sqrt{6}\|P^\eta(\vu_k)\|}{\varepsilon_k}\big)\Big\rceil, 
\end{align*}
where we also use that $\Bar{F}^\eta$ is $1$-strongly monotone and $(\eta L + 1)$-average Lipschitz. 
So it suffices to choose $M_k = \big\lceil 56\max\big\{n, \sqrt{n}(\eta L + 1)\big\}\log\big(1.252(k + 2)\big)\big\rceil$ to reach the accuracy $\varepsilon_k = \frac{\|P^{\eta}(\vu_k)\|}{(k + 2)^4}$. 

Then combining with Assumptions~\ref{asp: asp1}~and~\ref{asp: asp2}, we have the assumptions of Lemmas~\ref{lemma:uk-rec}~and~\ref{lemma:eck-change} hold. By Eq.~\eqref{eq:eck-change} from Lemma~\ref{lemma:eck-change},  telescoping it from $k - 1$ to $1$ we obtain 
\begin{align*}
    \E[\gC_k] \leq \;& \E[\gC_1] + 16\|\vu_0 - \vu_*\|^2\sum_{i = 1}^{k - 1}\frac{1}{(i + 2)^2} \leq \E[\gC_1] + 7\|\vu_0 - \vu_*\|^2.
\end{align*}
Then we come to bound $\E[\gC_1]$. Recall that $P^\eta$ is $\frac{1}{2}$-cocoercive and $\vu_1 - \vu_0 = - \frac{1}{2}P^\eta(\vu_0) - \frac{1}{2}\ve_0$ by Eq.~\eqref{eq:uk-diff-1}, then we have 
\begin{align*}
    \;& \innp{P^\eta(\vu_1) - P^\eta(\vu_0), \vu_1 - \vu_0} \geq \frac{1}{2}\|P^\eta(\vu_1) - P^\eta(\vu_0)\|^2 \\
    \iff \;& \|P^\eta(\vu_1)\|^2 \leq \innp{P^\eta(\vu_1), P^\eta(\vu_0)} - \innp{\ve_0, P^\eta(\vu_1) - P^\eta(\vu_0)}, 
\end{align*}
which leads to 
\begin{align*}
    \gC_1 = \;& \|P^\eta(\vu_1)\|^2 + 2\innp{P^\eta(\vu_1), \vu_1 - \vu_0} \\
    = \;& \|P^\eta(\vu_1)\|^2 - \innp{P^\eta(\vu_1), P^\eta(\vu_0) + \ve_0} \\
    \leq \;& \innp{\ve_0, P^\eta(\vu_0) - 2P^\eta(\vu_1)}. 
\end{align*}
Note that 
\begin{align*}
    \innp{\ve_0, P^\eta(\vu_0) - 2P^\eta(\vu_1)} = \;& 2\innp{\ve_0, P^\eta(\vu_0) - P^\eta(\vu_1)} - \innp{\ve_0, P^\eta(\vu_0)} \\
    \overset{(\romannumeral1)}{\leq} \;& 3\|\ve_0\|^2 + \frac{1}{2}\|P^\eta(\vu_0) - P^\eta(\vu_1)\|^2 + \frac{1}{4}\|P^\eta(\vu_0)\|^2 \\
    \overset{(\romannumeral2)}{\leq} \;& 3\|\ve_0\|^2 + 2\|\vu_0 - \vu_1\|^2 + \frac{1}{4}\|P^\eta(\vu_0)\|^2.
\end{align*}
where we use Young's inequality for $(\romannumeral1)$ and $(\romannumeral2)$ is due to that $P^c$ is $2$-Lipschitz. For the second term in the above inequality, we use $\vu_1 - \vu_0 = - \frac{1}{2}P^\eta(\vu_0) - \frac{1}{2}\ve_0$ and Young's inequality  again, and obtain 
\begin{align*}
    2\|\vu_0 - \vu_1\|^2 = \frac{1}{2}\|P^\eta(\vu_0) + \ve_0\|^2 \leq \frac{3}{2}\|\ve_0\|^2 + \frac{3}{4}\|P^\eta(\vu_0)\|^2, 
\end{align*}
thus 
\begin{align*}
    \innp{\ve_0, P^\eta(\vu_0) - 2P^\eta(\vu_1)} 
    \leq \frac{9}{2}\|\ve_0\|^2 + \|P^\eta(\vu_0)\|^2.
\end{align*}
Taking expectation w.r.t.\ all randomness on bothsides and noticing that $\E[\|\ve_0\|^2] \leq \frac{\|P^\eta(\vu_0)\|^2}{27}$ and $\|P^\eta(\vu_0)\|^2 \leq 4\|\vu_0 - \vu_*\|^2$, we have 
\begin{align*}
    \E[\innp{\ve_0, P^\eta(\vu_0) - 2P^c(\vu_1)}] \leq 5\|\vu_0 - \vu_*\|^2, 
\end{align*}
which leads to a bound on $\E[\gC_1]$ as follows 
\begin{align*}
    \E[\gC_1] \leq \E[\innp{\ve_0, P^\eta(\vu_0) - 2P^\eta(\vu_1)}] \leq 5\|\vu_0 - \vu_*\|^2. 
\end{align*}
So we obtain 
\begin{align*}
    \E[\gC_k] \leq \E[\gC_1] + 7\|\vu_0 - \vu_*\|^2 \leq 12\|\vu_0 - \vu_*\|^2.
\end{align*}
On the other hand, since $P^\eta$ is monotone and $P^\eta(\vu_*) = \mathbf{0}$, we have 
\begin{align*}
    \innp{P^\eta(\vu_k), \vu_k - \vu_0} = \;& \innp{P^\eta(\vu_k) - P^\eta(\vu_*), \vu_k - \vu_*} + \innp{P^\eta(\vu_k), \vu_* - \vu_0} \\
    \geq \;& \innp{P^\eta(\vu_k), \vu_* - \vu_0} \\
    \overset{(\romannumeral1)}{\geq} \;& -\|P^\eta(\vu_k)\|\|\vu_0 - \vu_*\|, 
\end{align*}
where we use Cauchy-Schwarz inequality for $(\romannumeral1)$. So we have 
\begin{align*}
    \E[\gC_k] \geq \frac{k(k + 1)}{2}\E[\|P^\eta(\vu_k)\|^2] - (k + 1)\|\vu_0 - \vu_*\|\E[\|P^\eta(\vu_k)\|], 
\end{align*}
which leads to 
\begin{align*}
    \frac{k(k + 1)}{2}\E[\|P^\eta(\vu_k)\|^2] \leq (k + 1)\|\vu_0 - \vu_*\|\E[\|P^\eta(\vu_k)\|] + 12\|\vu_0 - \vu_*\|^2. 
\end{align*}
By Jensen's inequality we have $\E[\|P^\eta(\vu_k)\|] \leq \big(\E[\|P^\eta(\vu_k)\|^2]\big)^{1/2}$, thus 
\begin{align*}
    \frac{k(k + 1)}{2}\E[\|P^\eta(\vu_k)\|^2] \leq (k + 1)\|\vu_0 - \vu_*\|\big(\E[\|P^\eta(\vu_k)\|^2]\big)^{1/2} + 12\|\vu_0 - \vu_*\|^2.
\end{align*}
By the larger root of this quadratic inequality w.r.t.\ $\big(\E[\|P^\eta(\vu_k)\|^2]\big)^{1/2}$, we obtain for $k \geq 1$
\begin{align*}
    \big(\E[\|P^\eta(\vu_k)\|^2]\big)^{1/2} \leq \;& \frac{\|\vu_0 - \vu_*\|}{k} + \sqrt{\frac{\|\vu_0 - \vu_*\|^2}{k^2} + \frac{24\|\vu_0 - \vu_*\|^2}{k(k + 1)}} \\
    \overset{(\romannumeral1)}{\leq} \;& \frac{2\|\vu_0 - \vu_*\|}{k} + \frac{2\sqrt{6}\|\vu_0 - \vu_*\|}{\sqrt{k(k + 1)}} \\
    \leq \;& \frac{7\|\vu_0 - \vu_*\|}{k}. 
\end{align*}
where $(\romannumeral1)$ is due to the fact that $\sqrt{a + b} \leq \sqrt{a} + \sqrt{b}$.

Hence, to guarantee that Algorithm~\ref{alg:monotone} returns a point $\vu_K$ such that $\E[\|P^\eta(\vu_K)\|^2] \leq \eta^2\varepsilon^2$, thus $\E[\|P^\eta(\vu_K)\|] \leq \eta\varepsilon$ by Jensen's inequality, we need $K = \big\lceil \frac{7\|\vu_0 - \vu_*\|}{\eta\varepsilon} \big\rceil$ outer iterations.

We now look at the cost of each inner loop to estimate $J_{\eta(F+G)}(\vu_k)$. 
Notice that for the subsolver $\mathtt{VR-FoRB}$, we take the number of inner iterations $M_k = \big\lceil 56\max\big\{n, \sqrt{n}(\eta L + 1)\big\}\log\big(1.252(k + 2)\big)\big\rceil$ with the constant per-iteration cost, so we have $\gO\big((n + \sqrt{n}(\eta L + 1))\log(k + 2)\big)$ oracle queries for $k$-th inner loop. 
Then for the total number of oracle queries, we note that 
\begin{align*}
    \sum_{k = 1}^{K}\big((n + \sqrt{n}(\eta L + 1))\log(k + 2)\big) = \gO\Big(\big(n + \sqrt{n}(\eta L + 1)\big)(K + 2)\log(K + 2)\Big). 
\end{align*}
Plugging in the choice of $K$ and suppressing the logarithm terms, we obtain $\Tilde{\gO}\big((n + \sqrt{n}(\eta L + 1))(\frac{\|\vu_0 - \vu_*\|}{\eta\varepsilon} + 1)\big)$ oracle complexity. Taking $\eta = \frac{\sqrt{n}}{L}$, we have the total oracle complexity to be $\widetilde{\gO}\big( n + \frac{\sqrt{n}L\|\vu_0 - \vu_*\|}{\varepsilon}\big)$. 
\end{proof}

\subsection{Estimation of the output}\label{subsec: postprocess}
\postprocess*
\begin{proof}
    We combine Lemma~\ref{lem : postprocess} and Theorem~\ref{thm:monotone}.
\end{proof}
\begin{lemma}\label{lem : postprocess}
    Let Assumptions~\ref{asp: asp1}~and~\ref{asp: asp2} hold and $\vu_k$ be such that $\mathbb{E}[\|P^\eta(\vu_k)\|] \leq \eta \varepsilon$
    with $\eta = \frac{\sqrt{n}}{L}$. Then, for $\vv_{\textup{out}}$ outputted by $\mathtt{VR{-}FoRB}(\vu_k, M, \mathrm{Id} + \eta(F + G) - \vu_k)$, we have that
    \begin{equation}
        \mathbb{E}[\mathrm{Res}_{F+G}(\vv_{\textup{out}})] \leq2\varepsilon,\label{eq: hyt9}
    \end{equation}
    where $M=\lceil 42(n+\sqrt{n})\log (19n) \rceil$ 
    and complexity of this step is $\mathcal{O}\left(n \log n \right)$. 
\end{lemma}
\begin{proof}
Let us denote $ \vu_k^* = J_{\eta(F+G)}(\vu_k)$, i.e., $\|P^\eta(\vu_k)\| = \|\vu_k - \vu_k^*\|\leq\eta\varepsilon$ and consider the uniform sampling for brevity. Note that we have in this case $A(\vu) = \eta F(\vu) + \vu - \vu_k$, $A_i(\vu) =  \eta F_i(\vu) + \vu - \vu_k$ and $B = \eta G$. 
By the update rule of $\mathtt{VR{-}FoRB}$ (where we use the index $t$ for the inner loop to prevent confusion), we have $\alpha = 1 - p$ and then for $t \geq 1$ 
\begin{align*}
    \vv_{t+1} + \tau B(\vv_{t+1}) \ni \;& (1-p)\vv_t + p\vw_t - \tau \tilde A(\vv_t),\\
    \iff A(\vv_{t+1}) + B(\vv_{t+1}) \ni \;& \frac{1 - p}{\tau}(\vv_t - \vv_{t+1}) + \frac{p}{\tau}(\vw_t - \vv_{t+1}) + A(\vv_{t+1}) - \tilde A(\vv_t)&
\end{align*}
where $\tilde A(\vv_t) = A(\vw_t) - A_i(\vw_{t-1}) + A_i(\vv_t)$ and we also have the implicit definition $\eta\vg_{t+1} = \frac{(1-p)}{\tau}\vv_t + \frac{p}{\tau}\vw_t - \tilde A(\vv_t) - \frac{1}{\tau} \vv_{t+1} \in B(\vv_{t+1})=\eta G(\vv_{t+1})$ since $B=\eta G$.

By using the definitions $A(\vu)=\eta F(\vu)+\vu-\vu_k, A_i(\vu)=\eta F_i(\vu) + \vu - \vu_k$, $\tilde A(\vv_t) = \eta F(\vw_t) + \vw_t - \vu_k - \eta F_i(\vw_{t-1})-\vw_{t-1}+\eta F_i(\vv_t) + \vv_t$ and rearranging we get
\begin{align*}
    \eta F(\vv_{t+1}) + \eta\vg_{t+1} = \;& \frac{1-p}{\tau}(\vv_t - \vv_{t+1}) + \frac{p}{\tau} (\vw_t - \vv_{t+1}) \\ 
    & + \eta F(\vv_{t+1}) - \eta F(\vw_t) \\ & + \eta F_i(\vw_{t-1}) - \eta F_i(\vv_t) \\
    & + (\vu_k - \vv_t) + (\vw_{t - 1} - \vw_{t}).
\end{align*}
Note that we have $\mathbb{E}[\|\vw_t - \vv_t\|] = (1-p)\mathbb{E}[\|\vw_{t-1}-\vv_t\|]$ and $\mathbb{E} [\| \vw_t - \vw_{t-1}\|] = p\mathbb{E}[\|\vv_t - \vw_{t-1}\|]$ by Alg.~\ref{alg: forb_vr_sc} and the tower rule.
As a result, triangle inequalities and Lipschitzness of $F$ 
give
\begin{align}
    \eta\mathbb{E}[\mathrm{Res}_{F+G}(\vv_{t+1})] \leq \;& \frac{1-p}{\tau}\mathbb{E}[\|\vv_t - \vv_{t+1} \|] + \big( \frac{p}{\tau} + \eta L_F \big)\mathbb{E}[\|\vv_{t+1} - \vw_{t}\|] \notag  \\
    & + \eta L\mathbb{E}[\|\vv_t - \vw_{t-1}\|] + \mathbb{E}[\|\vv_t - \vu_k^*\|] + \mathbb{E}[\| \vu_k - \vu_k^*\|] \notag \\
    & + \mathbb{E}[\|\vw_t - \vw_{t-1}\|] \notag \\
    \leq \;& \big(\frac{1}{\tau}+\eta L_F+1\big)\mathbb{E}[\|\vv_t - \vv_{t+1} \|] \notag \\
    & + \big( \frac{(1-p)p}{\tau} + (1-p)\eta L_F + \eta L + p  \big)\mathbb{E}[\|\vv_{t} - \vw_{t-1}\|] \notag  \\
    & + \mathbb{E}[\|\vv_{t+1} - \vu_k^*\|] + \mathbb{E}[\| \vu_k - \vu_k^*\|] \notag \\
    \leq \;& \left(2\tau^{-1} +3\eta L + 3 \right)\bar\varepsilon + \eta\varepsilon,\label{eq: cnbh4}
\end{align}
where the last step is because $L_F \leq L$ and given accuracy $\bar \varepsilon > 0$, the output of $\mathtt{VR{-}FoRB}$ gives for $t=M$
\begin{equation}
    \mathbb{E}[\|\vv_{t+1} - \vu_k^*\|^2] \leq \bar \varepsilon^2,~~~ \mathbb{E}[\| \vv_{t+1} - \vv_t\|^2] \leq \bar \varepsilon^2,~~~\mathbb{E}[\|\vv_{t} - \vw_{t-1}\|^2] \leq \bar \varepsilon^2.\label{eq: uir4}
\end{equation}
We now bound the oracle complexity and the number of iterations to get these bounds.

Denote by $\mathbb{E}_k$ the expectation conditioned on all the randomness up to and including $\vu_k$. Recall that $A$ is $(\eta L + 1)$-average Lipschitz and $A + B$ is $1$-strongly monotone, and we have our parameters to be $p = \frac{1}{n}$, $\eta = \frac{\sqrt{n}}{L}$ and $\tau\mu = \frac{\sqrt{p(1 - p)}}{2(\eta L + 1)} = \frac{\sqrt{n - 1}}{2n(\sqrt{n} + 1)} \leq \frac{1}{2n}$. Then following the same derivation from Eq.~\eqref{eq: grt3} to Eq.~\eqref{eq: sfr5} in the proof of Theorem~\ref{thm:strongly-monotone}, we first do not discard the last term on the right-hand side of Eq.~\eqref{eq: grt3}, and obtain (\emph{cf.}~\eqref{eq: slo4}) 
\begin{align}\label{eq: post-1}
    (1 - p + 4\tau\mu) a_{t+1} + b_{t+1} \leq (1-p) a_t + b_t - \frac{p}{2} \E[\|\vv_t - \vw_{t-1}\|^2] - \frac{1 - p}{2}\E[\|\vv_{t + 1} - \vv_t\|^2].
\end{align}
Further, without using $p\leq 1$ in~\eqref{eq: vbn4}, instead of~\eqref{eq: sfr5}, we have
\begin{align*}
    (1-p)a_t + b_t \leq \;& (1-p+6c)a_t + (1-c)b_t + \left(\frac{3cp}{2} + 2c\right) \mathbb{E}[\| \vv_{t} - \vw_{t-1}\|^2] \\
    \leq \;& (1-p+6c)a_t + (1-c)b_t + \frac{3p^2 + 4p}{14} \mathbb{E}[\| \vv_{t} - \vw_{t-1}\|^2],
\end{align*}
where the last line used $c\leq \frac{p}{7}$.
Combining with~\eqref{eq: post-1} gives (\emph{cf.}~\eqref{eq: nhf4})
\begin{align}
    (1 - p + 4\tau\mu)a_{t+1} + b_{t+1} \leq \;& (1-p+6c) a_t + (1-c)b_t \notag\\
    & - \frac{3p(1-p)}{14}\mathbb{E}[\|\vv_t- \vw_{t-1}\|^2] - \frac{1-p}{2} \mathbb{E}[\| \vv_{t+1} - \vv_t\|^2].
\end{align}
With the same derivation as obtaining~\eqref{eq: vvf4}, we have after using $4\tau\mu\geq 0$ and $4\tau\mu \leq 2p$
\begin{align*}
    \;& \frac{1-p}{2}\mathbb{E}_k[\|\vv_{t+1} - \vv_t\|^2] + \frac{3p(1-p)}{14}\mathbb{E}_k[\|\vv_{t} - \vw_{t-1}\|^2] +\frac{1-p}{2}\mathbb{E}_k[\| \vv_{t+1} - \vu_k^*\|^2] \notag \\
    \leq \;& \frac{5}{2}\exp\Big(-\min\big\{ \frac{\tau\mu}{1 + 4\tau\mu}, c\big\}(t+1)\Big)\|\vv_0 - \vu_k^*\|^2,
\end{align*}
where we also use that the solution of the inner subproblem is $\vu_k^*$.
Unrolling the expectation and using that $\mathtt{VR{-}FoRB}$ is initialized as $\vv_0=\vu_k$, we have 
\begin{align}
\;& \frac{3p(1-p)}{14}\left( \mathbb{E}[\|\vv_{t+1} - \vv_t\|^2] + \mathbb{E}[\|\vv_{t} - \vw_{t-1}\|^2] +\mathbb{E}[\| \vv_{t+1} - \vu_k^*\|^2]\right) \notag \\
\leq \;& \frac{5}{2}\exp\Big(-\min\big\{ \frac{\tau\mu}{1+4\tau\mu}, c\big\}(t+1)\Big)\mathbb{E}[\|\vu_k -\vu_k^*\|^2] \notag \\
\leq \;& \frac{5}{2}\exp\Big(-\min\big\{ \frac{\tau\mu}{1+4\tau\mu}, c\big\}(t+1)\Big)\eta^2\varepsilon^2,
\end{align}
where the last step is by $\mathbb{E}[\|\vu_k - J_{\eta(F+G)}(\vu_k)\|^2] = \mathbb{E}[\|\vu_k - \vu_k^*\|^2] \leq \eta^2\varepsilon^2$.
Hence, as in the end of Theorem~\ref{thm:strongly-monotone}, we have that~\eqref{eq: uir4} holds in $\big\lceil 14\max\left\{n, \sqrt{n}(\eta L + 1) \right\}\log(\frac{35\eta^2\varepsilon^2}{3p(1-p)\bar \varepsilon^2})\big\rceil$ iterations and with complexity $\mathcal{O}\Big( n\log(\frac{n\varepsilon}{L\bar \varepsilon})\Big)$ as $\eta = \frac{\sqrt{n}}{L}$.
In particular, we choose $\bar \varepsilon = \frac{\eta \varepsilon}{2\tau^{-1}+3\eta L + 3} = \frac{\varepsilon}{\frac{4nL(\sqrt{n}+1)}{\sqrt{n(n-1)}} + 3L + \frac{3L}{\sqrt{n}}}$
on~\eqref{eq: cnbh4} we get~\eqref{eq: hyt9}, and also have $\frac{\varepsilon}{\bar\varepsilon} \leq 10\sqrt{n}L+6L$. Plugging in this value gives the number of iterations as $\lceil 42(n+\sqrt{n})\log (19n) \rceil$ and overall complexity as $\mathcal{O}(n\log n)$.
\end{proof}

\postproc*
\begin{proof}
By definition, we have that $\eta F$ is maximally $\frac{\rho}{\eta}$-cohypomonotone.
When $\frac{\rho}{\eta} \leq \frac{1}{2}$, we have that $J_{\eta F}$ is single-valued and nonexpansive \cite[Prop. 3.7, Thm. 2.17]{bauschke2021generalized}. 
Subproblem in this case is finding $\bar{\vu}$ such that 
\begin{equation*}
    \mathbf{0} \in (\mathrm{Id} + \eta F)(\bar{\vu}) - \vu_k.
\end{equation*}
Since $F$ is $\rho$-cohypomonotone and $L_F$-Lipschitz, we have that 
\begin{equation*}
    \langle F(\vu) - F(\vv), \vu - \vv \rangle \geq -\rho L_F^2 \| \vu - \vv\|^2,
\end{equation*}
i.e.,\ $\eta F$ is $\rho \eta L_F^2$-hypomonotone (see \cite[Ex. 12.28]{rockafellar2009variational} for the definition).
As a result, our subproblem is $(1-\rho \eta L_F^2)$ strongly monotone and $(\eta L + 1)$-Lipschitz in expectation. 

In summary, to ensure that the resolvent is nonexpansive, single-valued and the subproblem is strongly monotone, we require
\begin{equation*}
    \frac{\rho}{\eta} \leq \frac{1}{2} \text{~~~and~~~} \rho \eta L_F^2 < 1, 
\end{equation*}
Of course, these bounds are optimized with $\eta=\frac{\sqrt{2}}{L_F}$, which leads to the requirement $\rho < \frac{1}{\sqrt{2}L_F}$. However, this choice of $\eta$ does not give the best oracle complexity with finite-sum form. 
In particular, by using a standard deterministic extragradient algorithm as subsolver in our framework (see e.g., \cite[App A.3]{diakonikolas2020halpern} for a proof for extragradient, with additional adaptivity that one can drop for simplicity), this would give complexity $\widetilde{O}\left(\frac{nL_F}{\varepsilon}\left(\frac{1}{1-\sqrt{2}\rho L_F}\right)\right)$. Using a variance reduced solver with $\eta=\frac{\sqrt{2}}{L_F}$ does not improve this.

Hence, we pick $\eta = \frac{\sqrt{n}}{L}$ as before. We can then use the same estimations as in the proof of Theorem~\ref{thm:monotone} by only changing the strong monotonicity parameter for the inner subproblem which affects the complexity of the inner subsolver and hence the final complexity. 
\end{proof}

\subsection{Details about Remark~\ref{rmk:matix-game}}\label{subsec: rmk}
In our analyses of this section, we use the simplified assumption that computation of $F_i$ is $n$ times cheaper than $F=\frac{1}{n}\sum_{i=1}^n F_i$ which gave the choice $p=\frac{1}{n}$. This is the most natural assumption given a generic $F$. This was the setting also in previous works such as~\citet{carmon2019variance,alacaoglu2021stochastic} when dealing with a general $F$ with a finite sum form. On the other hand, our bounds could have been also written in terms of $p$ (the probability for full operator evaluations for $\vw_{k+1}$ in Alg.~\ref{alg: forb_vr_sc}), as in~\citet{alacaoglu2021stochastic}. In this case, the general form for $p$ in terms of the costs of $F$ and $F_i$, denoted for simplicity as $\text{Cost}(F)$ and $\text{Cost}(F_i)$  would be $p=\frac{\text{Cost}(F_i)}{\text{Cost}(F)}$. 

However, for specific examples such as matrix games or linearly constrained optimization, we would use $p = \frac{m_1+m_2}{2m_1m_2}$ given a dense matrix $\mA$. We refer to~\citet{carmon2019variance,palaniappan2016stochastic,alacaoglu2021stochastic} for more details about this representation. This choice gives rise to the claimed complexity improvements in Remark~\ref{rmk:matix-game}. 

\section{Experiment Details}\label{appx:exp}
In this section, we provide further details about our experiment setup. For the matrix game case (also mentioned in Remark~\ref{rmk:matix-game}), we solve the problem 
\begin{equation}\notag 
\min_{\vx \in \sR^{m_1}} \max_{\vy \in \sR^{m_2}} \innp{\mA\vx, \vy} + \delta_{\Delta^{m_1}}(\vx) + \delta_{\Delta^{m_2}}(\vy)
\end{equation}
for $\mA \in \mathbb{R}^{m_1 \times m_2}$, the simplices $\Delta^{m_1}$, $\Delta^{m_2}$, where $\delta$ is the indicator function. We use the policeman and burglar matrix from~\citet{nemirovski2013mini} with $m_1 = m_2 = 500$, where the entries are given by $A_{ij} = \vz_i(1 - \exp(-\theta|i - j|))$ with $\theta = 0.8$ and $\vz \sim \gN(\vzero, I_{m_1})$. For the computation of $J_{\eta G}$, which corresponds to projection onto the simplex in this case, we use~\cite[Algorithm 1]{condat2016fast}. 

For the test case of Lagrangian of a quadratic program, we use the saddle function from~\citet{ouyang2021lower} as follows 
\begin{align*}
    \min_{\vx \in \R^{m_1}}\max_{\vy \in \R^{m_2}} \frac{1}{2}\vx^\top \mH \vx - \vh^\top\vx - \innp{\mA\vx - \vb, \vy}, 
\end{align*}
where $m_1 = m_2 = 200$, $\mH = 2\mA^\top \mA$ and 
\begin{align*}
\mA = \frac{1}{4}\begin{bmatrix}
    & & & -1 & 1 \\
    & & \cdots & \cdots & \\
    & -1 & 1 & & \\
    -1 & 1 & & & \\
    1 & & & &
\end{bmatrix} \in \R^{m_1 \times m_2},
\quad 
\vb = \frac{1}{4}\begin{bmatrix}
    1 \\
    1 \\
    \cdots \\
    1 \\
    1
\end{bmatrix} \in \R^{m_1},
\quad 
\vh = \frac{1}{4}\begin{bmatrix}
    0 \\
    0 \\
    \cdots \\
    0 \\
    1
\end{bmatrix} \in \R^{m_1}.
\end{align*}
Also, for Alg.~\ref{alg:monotone}, we directly measure the residual on its output point $\vu_K$, without doing another approximation step as in Corollary~\ref{cor: postprocess} for simplicity. This is guaranteed by $\E[\|P^\eta(\vu_k)\|] = \E[\|\vu_k - J_{\eta(F + G)}(\vu_k)\|] = \gO(1/k)$, and $\vu_k$ can be a good empirical approximation of $J_{\eta(F + G)}(\vu_k)$ as the algorithm proceeds.

\section{Conclusion and Perspectives}\label{app: perspectives}
It is worth pointing out that our results and the open question posed in Section~\ref{sec: conclusion} closely resemble the recent development of improved duality gap guarantees for finite-sum monotone VIs with variance reduction. In particular, it was the work of~\citet{palaniappan2016stochastic} that provided the first variance reduced variational inequality algorithm which was an indirect procedure based on the Catalyst proximal point framework and forward-backward algorithm. This work already showed the benefit of variance reduction compared to deterministic algorithms for strongly monotone inclusions and monotone VIs by using standard reductions using regularization. 

A direct algorithm for the important special case of matrix games was provided by~\citet{carmon2019variance}. At the same time, this work also handled the general monotone VI case with an indirect approach. Other direct algorithms, given in~\citet{chavdarova2019reducing} for the strongly monotone case and in~\citet{alacaoglu2021forward} for the monotone case,  were simple but did not improve the complexity bounds compared to deterministic algorithms. It was the work of~\cite{alacaoglu2021stochastic} that obtained direct and single-loop algorithms with complexity improvements for general monotone VIs, nearly five years after the indirect result of~\citet{palaniappan2016stochastic}. An alternative direct approach focusing on finite-sum saddle point problems was also studied in~\cite{yazdandoost2023stochastic}. Soon after the direct algorithms, matching lower bounds for the duality gap were also provided for finite-sum monotone VIs~\citep{han2021lower}.

In this context, our results for the monotone Lipschitz case provided the first improvement with variance reduction for residual guarantees. This can be seen as corresponding to the results of~\citet{palaniappan2016stochastic,carmon2019variance} that had indirect algorithms with complexity improvements for the duality gap for monotone VIs. What remains to be done to complete the picture for finite-sum monotone inclusions and monotone VIs is developing direct algorithms  with tight complexity guarantees for the residual, similar to the process that we have seen for duality gap guarantees.

\end{document}

%% file: math_commands.tex
\usepackage{amsmath,amsfonts,bm}
\usepackage{amssymb}
\usepackage{amsthm}
\usepackage{graphicx}
\usepackage{caption}
\usepackage{color}
\usepackage{comment}
\usepackage{enumitem}
\usepackage{mathtools}
\usepackage{thmtools, thm-restate}
\usepackage{algorithm}
\usepackage[noend]{algorithmic}
\usepackage{verbatim}

\newcommand{\innp}[1]{\left\langle #1 \right\rangle}
\newcommand{\tildeF}{\widetilde{F}}

\theoremstyle{plain} \numberwithin{equation}{section}
\newtheorem{theorem}{Theorem}[section]
\numberwithin{theorem}{section}

\newtheorem{lemma}[theorem]{Lemma}

\theoremstyle{definition}

\newtheorem{remark}[theorem]{Remark}

\newtheorem{assumption}{Assumption}

\def\1{\bm{1}}

\def\vzero{{\bm{0}}}

\def\vb{{\bm{b}}}

\def\ve{{\mathbf{e}}}

\def\vg{{\mathbf{g}}}
\def\vh{{\bm{h}}}

\def\vu{{\mathbf{u}}}
\def\vv{{\mathbf{v}}}
\def\vw{{\mathbf{w}}}
\def\vx{{\mathbf{x}}}
\def\vy{{\mathbf{y}}}
\def\vz{{\bm{z}}}

\def\mA{{\bm{A}}}

\def\mH{{\bm{H}}}

\DeclareMathAlphabet{\mathsfit}{\encodingdefault}{\sfdefault}{m}{sl}
\SetMathAlphabet{\mathsfit}{bold}{\encodingdefault}{\sfdefault}{bx}{n}

\def\gC{{\mathcal{C}}}

\def\gF{{\mathcal{F}}}

\def\gN{{\mathcal{N}}}
\def\gO{{\mathcal{O}}}
\def\gP{{\mathcal{P}}}

\def\gS{{\mathcal{S}}}
\def\gT{{\mathcal{T}}}

\def\sR{{\mathbb{R}}}

\newcommand{\E}{\mathbb{E}}

\newcommand{\R}{\mathbb{R}}



%% file: arxiv_main.bbl
\begin{thebibliography}{64}
\providecommand{\natexlab}[1]{#1}
\providecommand{\url}[1]{\texttt{#1}}
\expandafter\ifx\csname urlstyle\endcsname\relax
  \providecommand{\doi}[1]{doi: #1}\else
  \providecommand{\doi}{doi: \begingroup \urlstyle{rm}\Url}\fi

\bibitem[Alacaoglu \& Malitsky(2022)Alacaoglu and
  Malitsky]{alacaoglu2021stochastic}
Ahmet Alacaoglu and Yura Malitsky.
\newblock Stochastic variance reduction for variational inequality methods.
\newblock In \emph{Proc.~COLT'22}, 2022.

\bibitem[Alacaoglu et~al.(2021)Alacaoglu, Malitsky, and
  Cevher]{alacaoglu2021forward}
Ahmet Alacaoglu, Yura Malitsky, and Volkan Cevher.
\newblock Forward-reflected-backward method with variance reduction.
\newblock \emph{Computational Optimization and Applications}, 80\penalty0
  (2):\penalty0 321--346, 2021.

\bibitem[Allen-Zhu(2017)]{allen2017katyusha}
Zeyuan Allen-Zhu.
\newblock Katyusha: The first direct acceleration of stochastic gradient
  methods.
\newblock \emph{The Journal of Machine Learning Research}, 18\penalty0
  (1):\penalty0 8194--8244, 2017.

\bibitem[Allen-Zhu(2018)]{allen2018make}
Zeyuan Allen-Zhu.
\newblock How to make the gradients small stochastically: Even faster convex
  and nonconvex sgd.
\newblock In \emph{Proc.~NeurIPS'18}, 2018.

\bibitem[Bauschke \& Combettes(2011)Bauschke and Combettes]{bauschke2011convex}
Heinz~H Bauschke and Patrick~L Combettes.
\newblock \emph{Convex analysis and monotone operator theory in Hilbert
  spaces}, volume 408.
\newblock Springer, 2011.

\bibitem[Bauschke et~al.(2021)Bauschke, Moursi, and
  Wang]{bauschke2021generalized}
Heinz~H Bauschke, Walaa~M Moursi, and Xianfu Wang.
\newblock Generalized monotone operators and their averaged resolvents.
\newblock \emph{Mathematical Programming}, 189:\penalty0 55--74, 2021.

\bibitem[Beznosikov et~al.(2023)Beznosikov, Gorbunov, Berard, and
  Loizou]{beznosikov2023stochastic}
Aleksandr Beznosikov, Eduard Gorbunov, Hugo Berard, and Nicolas Loizou.
\newblock Stochastic gradient descent-ascent: Unified theory and new efficient
  methods.
\newblock In \emph{Proc.~AISTATS'23}, 2023.

\bibitem[Cai et~al.(2022{\natexlab{a}})Cai, Song, Guzm{\'a}n, and
  Diakonikolas]{cai2022stochastic}
Xufeng Cai, Chaobing Song, Crist{\'o}bal Guzm{\'a}n, and Jelena Diakonikolas.
\newblock Stochastic halpern iteration with variance reduction for stochastic
  monotone inclusions.
\newblock In \emph{Proc.~NeurIPS'22}, 2022{\natexlab{a}}.

\bibitem[Cai \& Zheng(2023)Cai and Zheng]{cai2022acceleratedsingle}
Yang Cai and Weiqiang Zheng.
\newblock Accelerated single-call methods for constrained min-max optimization.
\newblock In \emph{Proc.~ICLR'23}, 2023.

\bibitem[Cai et~al.(2022{\natexlab{b}})Cai, Oikonomou, and
  Zheng]{cai2022accelerated}
Yang Cai, Argyris Oikonomou, and Weiqiang Zheng.
\newblock Accelerated algorithms for monotone inclusions and constrained
  nonconvex-nonconcave min-max optimization.
\newblock \emph{arXiv preprint arXiv:2206.05248}, 2022{\natexlab{b}}.

\bibitem[Carmon et~al.(2019)Carmon, Jin, Sidford, and Tian]{carmon2019variance}
Yair Carmon, Yujia Jin, Aaron Sidford, and Kevin Tian.
\newblock Variance reduction for matrix games.
\newblock \emph{Proc.~NeurIPS'19}, 32, 2019.

\bibitem[Chavdarova et~al.(2019)Chavdarova, Gidel, Fleuret, and
  Lacoste-Julien]{chavdarova2019reducing}
Tatjana Chavdarova, Gauthier Gidel, Fran{\c{c}}ois Fleuret, and Simon
  Lacoste-Julien.
\newblock Reducing noise in gan training with variance reduced extragradient.
\newblock \emph{Proc.~NeurIPS'19}, 32, 2019.

\bibitem[Chen \& Luo(2022)Chen and Luo]{chen2022near}
Lesi Chen and Luo Luo.
\newblock Near-optimal algorithms for making the gradient small in stochastic
  minimax optimization.
\newblock \emph{arXiv preprint arXiv:2208.05925}, 2022.

\bibitem[Choudhury et~al.(2023)Choudhury, Gorbunov, and
  Loizou]{choudhury2023single}
Sayantan Choudhury, Eduard Gorbunov, and Nicolas Loizou.
\newblock Single-call stochastic extragradient methods for structured
  non-monotone variational inequalities: Improved analysis under weaker
  conditions.
\newblock \emph{arXiv preprint arXiv:2302.14043}, 2023.

\bibitem[Combettes \& Pennanen(2004)Combettes and
  Pennanen]{combettes2004proximal}
Patrick~L Combettes and Teemu Pennanen.
\newblock Proximal methods for cohypomonotone operators.
\newblock \emph{SIAM journal on control and optimization}, 43\penalty0
  (2):\penalty0 731--742, 2004.

\bibitem[Condat(2016)]{condat2016fast}
Laurent Condat.
\newblock Fast projection onto the simplex and the $l_1$ ball.
\newblock \emph{Mathematical Programming}, 158\penalty0 (1-2):\penalty0
  575--585, 2016.

\bibitem[Davis(2023)]{davis2023variance}
Damek Davis.
\newblock Variance reduction for root-finding problems.
\newblock \emph{Mathematical Programming}, 197\penalty0 (1):\penalty0 375--410,
  2023.

\bibitem[Defazio et~al.(2014)Defazio, Bach, and
  Lacoste-Julien]{defazio2014saga}
Aaron Defazio, Francis Bach, and Simon Lacoste-Julien.
\newblock {SAGA}: A fast incremental gradient method with support for
  non-strongly convex composite objectives.
\newblock In \emph{Proc.~NeurIPS'14}, 2014.

\bibitem[Diakonikolas(2020)]{diakonikolas2020halpern}
Jelena Diakonikolas.
\newblock Halpern iteration for near-optimal and parameter-free monotone
  inclusion and strong solutions to variational inequalities.
\newblock In \emph{Proc.~COLT'20}, 2020.

\bibitem[Facchinei \& Pang(2003)Facchinei and Pang]{facchinei2007finite}
Francisco Facchinei and Jong-Shi Pang.
\newblock \emph{Finite-dimensional variational inequalities and complementarity
  problems}.
\newblock Springer Science \& Business Media, 2003.

\bibitem[Fang et~al.(2018)Fang, Li, Lin, and Zhang]{fang2018spider}
Cong Fang, Chris~Junchi Li, Zhouchen Lin, and Tong Zhang.
\newblock Spider: Near-optimal non-convex optimization via stochastic
  path-integrated differential estimator.
\newblock \emph{Proc.~NeurIPS'18}, 2018.

\bibitem[Golowich et~al.(2020)Golowich, Pattathil, Daskalakis, and
  Ozdaglar]{golowich2020last}
Noah Golowich, Sarath Pattathil, Constantinos Daskalakis, and Asuman Ozdaglar.
\newblock Last iterate is slower than averaged iterate in smooth convex-concave
  saddle point problems.
\newblock In \emph{Proc.~COLT'20}, 2020.

\bibitem[Goodfellow et~al.(2014)Goodfellow, Pouget-Abadie, Mirza, Xu,
  Warde-Farley, Ozair, Courville, and Bengio]{goodfellow2014generative}
Ian Goodfellow, Jean Pouget-Abadie, Mehdi Mirza, Bing Xu, David Warde-Farley,
  Sherjil Ozair, Aaron Courville, and Yoshua Bengio.
\newblock Generative adversarial nets.
\newblock In \emph{Proc.~NeurIPS'14}, 2014.

\bibitem[Gorbunov et~al.(2022)Gorbunov, Berard, Gidel, and
  Loizou]{gorbunov2022stochastic}
Eduard Gorbunov, Hugo Berard, Gauthier Gidel, and Nicolas Loizou.
\newblock Stochastic extragradient: General analysis and improved rates.
\newblock In \emph{Proc.~AISTATS'22}, 2022.

\bibitem[Gower et~al.(2020)Gower, Schmidt, Bach, and
  Richt{\'a}rik]{gower2020variance}
Robert~M Gower, Mark Schmidt, Francis Bach, and Peter Richt{\'a}rik.
\newblock Variance-reduced methods for machine learning.
\newblock \emph{Proceedings of the IEEE}, 108\penalty0 (11):\penalty0
  1968--1983, 2020.

\bibitem[Halpern(1967)]{halpern1967fixed}
Benjamin Halpern.
\newblock Fixed points of nonexpanding maps.
\newblock \emph{Bulletin of the American Mathematical Society}, 73\penalty0
  (6):\penalty0 957--961, 1967.

\bibitem[Han et~al.(2021)Han, Xie, and Zhang]{han2021lower}
Yuze Han, Guangzeng Xie, and Zhihua Zhang.
\newblock Lower complexity bounds of finite-sum optimization problems: The
  results and construction.
\newblock \emph{arXiv preprint arXiv:2103.08280}, 2021.

\bibitem[Hofmann et~al.(2015)Hofmann, Lucchi, Lacoste-Julien, and
  McWilliams]{hofmann2015variance}
Thomas Hofmann, Aurelien Lucchi, Simon Lacoste-Julien, and Brian McWilliams.
\newblock Variance reduced stochastic gradient descent with neighbors.
\newblock In \emph{Proc.~NeurIPS'15}, 2015.

\bibitem[Hsieh et~al.(2019)Hsieh, Iutzeler, Malick, and
  Mertikopoulos]{hsieh2019convergence}
Yu-Guan Hsieh, Franck Iutzeler, J{\'e}r{\^o}me Malick, and Panayotis
  Mertikopoulos.
\newblock On the convergence of single-call stochastic extra-gradient methods.
\newblock In \emph{Proc.~NeurIPS'19}, 2019.

\bibitem[Johnson \& Zhang(2013)Johnson and Zhang]{johnson2013accelerating}
Rie Johnson and Tong Zhang.
\newblock Accelerating stochastic gradient descent using predictive variance
  reduction.
\newblock In \emph{Advances in Neural Information Processing Systems},
  volume~26, 2013.

\bibitem[Johnstone et~al.(2021)Johnstone, Eckstein, Flynn, and
  Yoo]{johnstone2021stochastic}
Patrick~R Johnstone, Jonathan Eckstein, Thomas Flynn, and Shinjae Yoo.
\newblock Stochastic projective splitting: Solving saddle-point problems with
  multiple regularizers.
\newblock \emph{arXiv preprint arXiv:2106.13067}, 2021.

\bibitem[Kohlenbach(2022)]{kohlenbach2022proximal}
Ulrich Kohlenbach.
\newblock On the proximal point algorithm and its halpern-type variant for
  generalized monotone operators in hilbert space.
\newblock \emph{Optimization Letters}, 16\penalty0 (2):\penalty0 611--621,
  2022.

\bibitem[Korpelevich(1977)]{korpelevich1977extragradient}
GM~Korpelevich.
\newblock Extragradient method for finding saddle points and other problems.
\newblock \emph{Matekon}, 13\penalty0 (4):\penalty0 35--49, 1977.

\bibitem[Kovalev \& Gasnikov(2022)Kovalev and Gasnikov]{kovalev2022first}
Dmitry Kovalev and Alexander Gasnikov.
\newblock The first optimal algorithm for smooth and
  strongly-convex-strongly-concave minimax optimization.
\newblock In \emph{Proc.~NeurIPS'22}, 2022.

\bibitem[Kovalev et~al.(2020)Kovalev, Horv{\'a}th, and
  Richt{\'a}rik]{kovalev2020don}
Dmitry Kovalev, Samuel Horv{\'a}th, and Peter Richt{\'a}rik.
\newblock Don’t jump through hoops and remove those loops: Svrg and katyusha
  are better without the outer loop.
\newblock In \emph{Proc.~ALT'20}, 2020.

\bibitem[Lan et~al.(2019)Lan, Li, and Zhou]{lan2019unified}
Guanghui Lan, Zhize Li, and Yi~Zhou.
\newblock A unified variance-reduced accelerated gradient method for convex
  optimization.
\newblock In \emph{Proc.~NeurIPS'19}, 2019.

\bibitem[Lee \& Kim(2021)Lee and Kim]{lee2021fast}
Sucheol Lee and Donghwan Kim.
\newblock Fast extra gradient methods for smooth structured
  nonconvex-nonconcave minimax problems.
\newblock In \emph{Proc.~NeurIPS'21}, volume~34, 2021.

\bibitem[Li et~al.(2021)Li, Bao, Zhang, and Richt{\'a}rik]{li2021page}
Zhize Li, Hongyan Bao, Xiangliang Zhang, and Peter Richt{\'a}rik.
\newblock Page: A simple and optimal probabilistic gradient estimator for
  nonconvex optimization.
\newblock In \emph{Proc.~ICML'21}, 2021.

\bibitem[Lohr(2021)]{lohr2021sampling}
Sharon~L Lohr.
\newblock \emph{Sampling: design and analysis}.
\newblock CRC press, 2021.

\bibitem[Loizou et~al.(2021)Loizou, Berard, Gidel, Mitliagkas, and
  Lacoste-Julien]{loizou2021stochastic}
Nicolas Loizou, Hugo Berard, Gauthier Gidel, Ioannis Mitliagkas, and Simon
  Lacoste-Julien.
\newblock Stochastic gradient descent-ascent and consensus optimization for
  smooth games: Convergence analysis under expected co-coercivity.
\newblock In \emph{Proc.~NeurIPS'21}, 2021.

\bibitem[Madry et~al.(2018)Madry, Makelov, Schmidt, Tsipras, and
  Vladu]{madry2018towards}
Aleksander Madry, Aleksandar Makelov, Ludwig Schmidt, Dimitris Tsipras, and
  Adrian Vladu.
\newblock Towards deep learning models resistant to adversarial attacks.
\newblock In \emph{Proc.~ICLR'18}, 2018.

\bibitem[Malitsky(2019)]{Malitsky2019}
Yura Malitsky.
\newblock Golden ratio algorithms for variational inequalities.
\newblock \emph{Mathematical Programming}, Jul 2019.

\bibitem[Malitsky \& Tam(2020)Malitsky and Tam]{malitsky2020forward}
Yura Malitsky and Matthew~K Tam.
\newblock A forward-backward splitting method for monotone inclusions without
  cocoercivity.
\newblock \emph{SIAM Journal on Optimization}, 30\penalty0 (2):\penalty0
  1451--1472, 2020.

\bibitem[Morin \& Giselsson(2022)Morin and Giselsson]{morin2022cocoercivity}
Martin Morin and Pontus Giselsson.
\newblock Cocoercivity, smoothness and bias in variance-reduced stochastic
  gradient methods.
\newblock \emph{Numerical Algorithms}, 91\penalty0 (2):\penalty0 749--772,
  2022.

\bibitem[Nemirovski(2004)]{nemirovski2004prox}
Arkadi Nemirovski.
\newblock Prox-method with rate of convergence {$O(1/t)$} for variational
  inequalities with {L}ipschitz continuous monotone operators and smooth
  convex-concave saddle point problems.
\newblock \emph{SIAM Journal on Optimization}, 15\penalty0 (1):\penalty0
  229--251, 2004.

\bibitem[Nemirovski(2013)]{nemirovski2013mini}
Arkadi Nemirovski.
\newblock Mini-course on convex programming algorithms.
\newblock \emph{Lecture Notes}, 2013.

\bibitem[Nesterov(2007)]{nesterov2007dual}
Yurii Nesterov.
\newblock Dual extrapolation and its applications to solving variational
  inequalities and related problems.
\newblock \emph{Mathematical Programming}, 109\penalty0 (2-3):\penalty0
  319--344, 2007.

\bibitem[Nguyen et~al.(2017)Nguyen, Liu, Scheinberg, and
  Tak{\'a}{\v{c}}]{nguyen2017sarah}
Lam~M Nguyen, Jie Liu, Katya Scheinberg, and Martin Tak{\'a}{\v{c}}.
\newblock Sarah: A novel method for machine learning problems using stochastic
  recursive gradient.
\newblock In \emph{Proc.~ICML'17}, 2017.

\bibitem[Ouyang \& Xu(2021)Ouyang and Xu]{ouyang2021lower}
Yuyuan Ouyang and Yangyang Xu.
\newblock Lower complexity bounds of first-order methods for convex-concave
  bilinear saddle-point problems.
\newblock \emph{Mathematical Programming}, 185\penalty0 (1-2):\penalty0 1--35,
  2021.

\bibitem[Palaniappan \& Bach(2016)Palaniappan and
  Bach]{palaniappan2016stochastic}
Balamurugan Palaniappan and Francis Bach.
\newblock Stochastic variance reduction methods for saddle-point problems.
\newblock In \emph{Proc.~NeurIPS'16}, 2016.

\bibitem[Park \& Ryu(2022)Park and Ryu]{park2022exact}
Jisun Park and Ernest~K Ryu.
\newblock Exact optimal accelerated complexity for fixed-point iterations.
\newblock In \emph{Proc.~ICML'22}, 2022.

\bibitem[Pethick et~al.(2023)Pethick, Fercoq, Latafat, Patrinos, and
  Cevher]{pethick2023solving}
Thomas Pethick, Olivier Fercoq, Puya Latafat, Panagiotis Patrinos, and Volkan
  Cevher.
\newblock Solving stochastic weak minty variational inequalities without
  increasing batch size.
\newblock \emph{arXiv preprint arXiv:2302.09029}, 2023.

\bibitem[Qian et~al.(2021)Qian, Qu, and Richt{\'a}rik]{qian2021svrg}
Xun Qian, Zheng Qu, and Peter Richt{\'a}rik.
\newblock L-svrg and l-katyusha with arbitrary sampling.
\newblock \emph{The Journal of Machine Learning Research}, 22\penalty0
  (1):\penalty0 4991--5039, 2021.

\bibitem[Rockafellar(1976)]{rockafellar1976monotone}
R~Tyrrell Rockafellar.
\newblock Monotone operators and the proximal point algorithm.
\newblock \emph{SIAM journal on control and optimization}, 14\penalty0
  (5):\penalty0 877--898, 1976.

\bibitem[Rockafellar \& Wets(2009)Rockafellar and
  Wets]{rockafellar2009variational}
R~Tyrrell Rockafellar and Roger J-B Wets.
\newblock \emph{Variational analysis}, volume 317.
\newblock Springer Science \& Business Media, 2009.

\bibitem[Schmidt et~al.(2017)Schmidt, Le~Roux, and Bach]{schmidt2017minimizing}
Mark Schmidt, Nicolas Le~Roux, and Francis Bach.
\newblock Minimizing finite sums with the stochastic average gradient.
\newblock \emph{Mathematical Programming}, 162\penalty0 (1):\penalty0 83--112,
  2017.

\bibitem[Song et~al.(2020)Song, Jiang, and Ma]{song2020variance}
Chaobing Song, Yong Jiang, and Yi~Ma.
\newblock Variance reduction via accelerated dual averaging for finite-sum
  optimization.
\newblock In \emph{Proc.~NeurIPS'20}, 2020.

\bibitem[Tran-Dinh(2023)]{tran2023sublinear}
Quoc Tran-Dinh.
\newblock Sublinear convergence rates of extragradient-type methods: A survey
  on classical and recent developments.
\newblock \emph{arXiv preprint arXiv:2303.17192}, 2023.

\bibitem[Tran-Dinh \& Luo(2021)Tran-Dinh and Luo]{tran2021halpern}
Quoc Tran-Dinh and Yang Luo.
\newblock Halpern-type accelerated and splitting algorithms for monotone
  inclusions.
\newblock \emph{arXiv preprint arXiv:2110.08150}, 2021.

\bibitem[Tran-Dinh \& Luo(2023)Tran-Dinh and Luo]{tran2023randomized}
Quoc Tran-Dinh and Yang Luo.
\newblock Randomized block-coordinate optimistic gradient algorithms for
  root-finding problems.
\newblock \emph{arXiv preprint arXiv:2301.03113}, 2023.

\bibitem[Yazdandoost~Hamedani \& Jalilzadeh(2023)Yazdandoost~Hamedani and
  Jalilzadeh]{yazdandoost2023stochastic}
Erfan Yazdandoost~Hamedani and Afrooz Jalilzadeh.
\newblock A stochastic variance-reduced accelerated primal-dual method for
  finite-sum saddle-point problems.
\newblock \emph{Computational Optimization and Applications}, pp.\  1--27,
  2023.

\bibitem[Yoon \& Ryu(2021)Yoon and Ryu]{Yoon2021OptimalGradientNorm}
Taeho Yoon and Ernest~K Ryu.
\newblock Accelerated algorithms for smooth convex-concave minimax problems
  with {$O(1/k^2)$} rate on squared gradient norm.
\newblock In \emph{Proc.~ICML'21}, 2021.

\bibitem[Zhang et~al.(2021)Zhang, Yang, and Ba{\c{s}}ar]{zhang2021multi}
Kaiqing Zhang, Zhuoran Yang, and Tamer Ba{\c{s}}ar.
\newblock Multi-agent reinforcement learning: A selective overview of theories
  and algorithms.
\newblock \emph{Handbook of Reinforcement Learning and Control}, pp.\
  321--384, 2021.

\bibitem[Zhou et~al.(2022)Zhou, Tian, So, and Cheng]{zhou2022practical}
Kaiwen Zhou, Lai Tian, Anthony Man-Cho So, and James Cheng.
\newblock Practical schemes for finding near-stationary points of convex
  finite-sums.
\newblock In \emph{Proc.~AISTATS'22}, 2022.

\end{thebibliography}
